\documentclass[conference]{IEEEtran}
\IEEEoverridecommandlockouts
\usepackage{times}

\usepackage[numbers]{natbib}
\usepackage{multicol}
\usepackage[bookmarks=true]{hyperref}

\usepackage{booktabs}
\usepackage{multirow}
\usepackage{amsmath,amssymb}
\usepackage{subfig}
\usepackage{graphicx}
\usepackage{xspace}

\newtheorem{theorem}{Theorem}[section]

\newtheorem{lemma}[theorem]{Lemma}

\newcommand{\method}{CLAMP\xspace}
\newcommand{\ba}{\mathbf{a}}

\begin{document}

\title{\method: Contrastive Learning for 3D \\Multi-View Action-Conditioned Robotic \\Manipulation Pretraining}

\author{\IEEEauthorblockN{I-Chun Arthur Liu\IEEEauthorrefmark{1}\IEEEauthorrefmark{2},
Krzysztof Choromanski\IEEEauthorrefmark{1},
Sandy Huang\IEEEauthorrefmark{1}, 
Connor Schenck\IEEEauthorrefmark{1}
\IEEEauthorblockA{\IEEEauthorrefmark{1}Google DeepMind \IEEEauthorrefmark{2}University of Southern California}
\thanks{\IEEEauthorrefmark{2}Work done as a student researcher at Google DeepMind.}
}}

\maketitle

\begin{abstract}
Leveraging pre-trained 2D image representations in behavior cloning policies has achieved great success and has become a standard approach for robotic manipulation. However, such representations fail to capture the 3D spatial information about objects and scenes that is essential for precise manipulation. In this work, we introduce Contrastive Learning for 3D Multi-View Action-Conditioned Robotic Manipulation Pretraining (\method), a novel 3D pre-training framework that utilizes point clouds and robot actions. From the merged point cloud computed from RGB-D images and camera extrinsics, we re-render multi-view four-channel image observations with depth and 3D coordinates, including dynamic wrist views, to provide clearer views of target objects for high-precision manipulation tasks. The pre-trained encoders learn to associate the 3D geometric and positional information of objects with robot action patterns via contrastive learning on large-scale simulated robot trajectories. During encoder pre-training, we pre-train a Diffusion Policy to initialize the policy weights for fine-tuning, which is essential for improving fine-tuning sample efficiency and performance. After pre-training, we fine-tune the policy on a limited amount of task demonstrations using the learned image and action representations. We demonstrate that this pre-training and fine-tuning design substantially improves learning efficiency and policy performance on unseen tasks. Furthermore, we show that \method outperforms state-of-the-art baselines across six simulated tasks and five real-world tasks. The project website and videos can be found at \href{https://clamp3d.github.io/CLAMP/}{https://clamp3d.github.io/CLAMP/}.
\end{abstract}

\IEEEpeerreviewmaketitle

\section{Introduction}

\begin{figure}[t]
    % \vspace{-0.5em}
    \centering
    \includegraphics[width=\columnwidth]{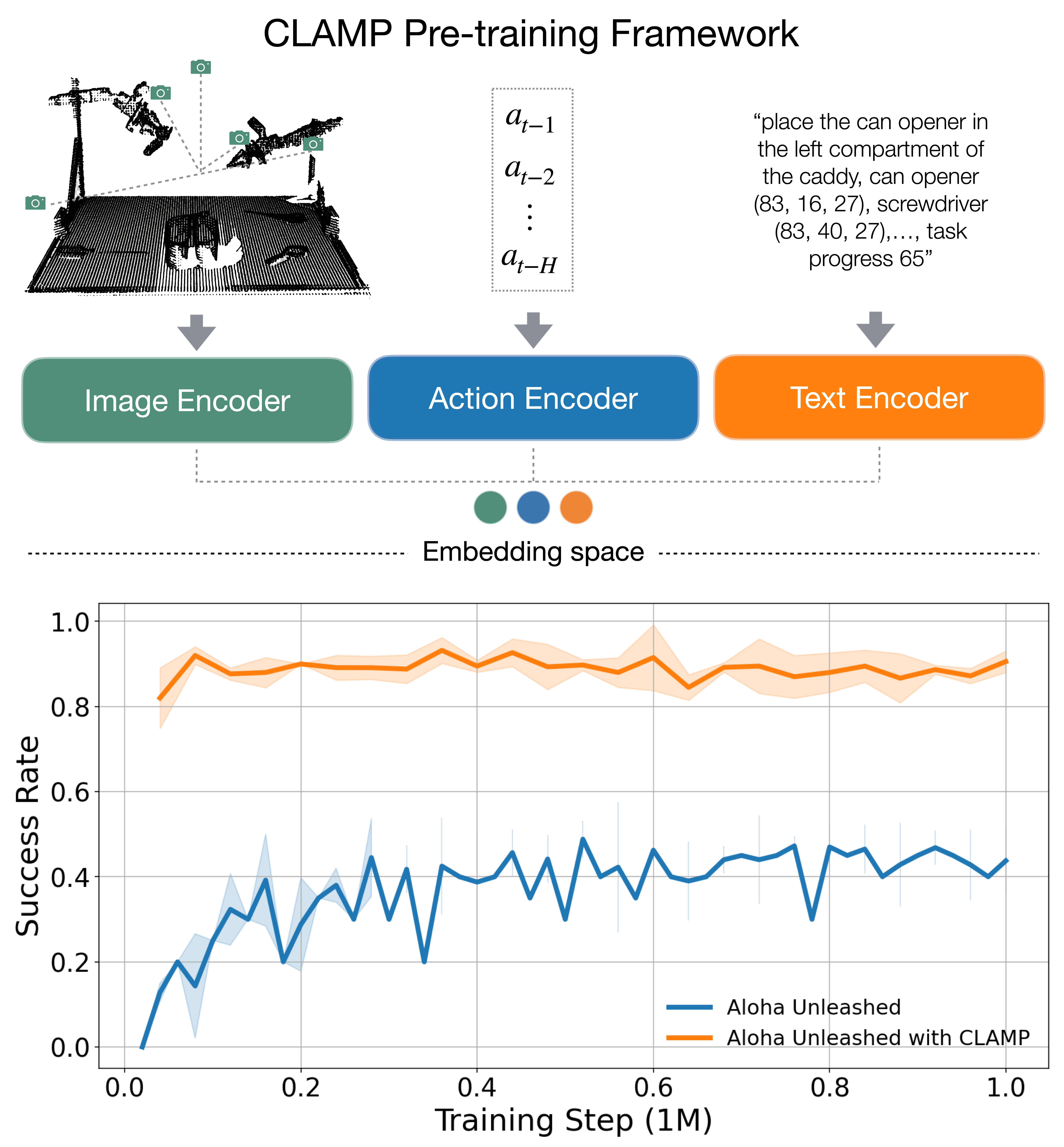}
    \caption{\method is a 3D pre-training framework for robotic manipulation. We contrastively pre-train an image encoder with fixed and dynamic multi-view observations, an action encoder with histories of action chunks, and a text encoder with spatial and temporal object- and task-level information. \method learns image and action representations that improve downstream policy fine-tuning efficiency and performance. The bottom figure compares the ALOHA Unleashed evaluation curves during fine-tuning without and with \method on the \texttt{Mug on Plate} task.}
    \label{fig:hero_figure}
    \vspace{-1.0em}
\end{figure}

Contrastive pre-training on internet-scale image–text pairs has become a standard approach for learning strong computer vision backbones \cite{radford2021learning,zhai2023sigmoid,maninis2024tips}. However, an analogous standard approach has not emerged in robotics due to the limited availability of large-scale robot data. Moreover, it remains unclear which modalities and image representations, such as RGB, point clouds, or voxel grids, are best suited for pre-training. Pre-training in robotics is further complicated by high-dimensional state and action spaces.

\begin{figure*}[t]
    \centering
    \includegraphics[width=\textwidth]{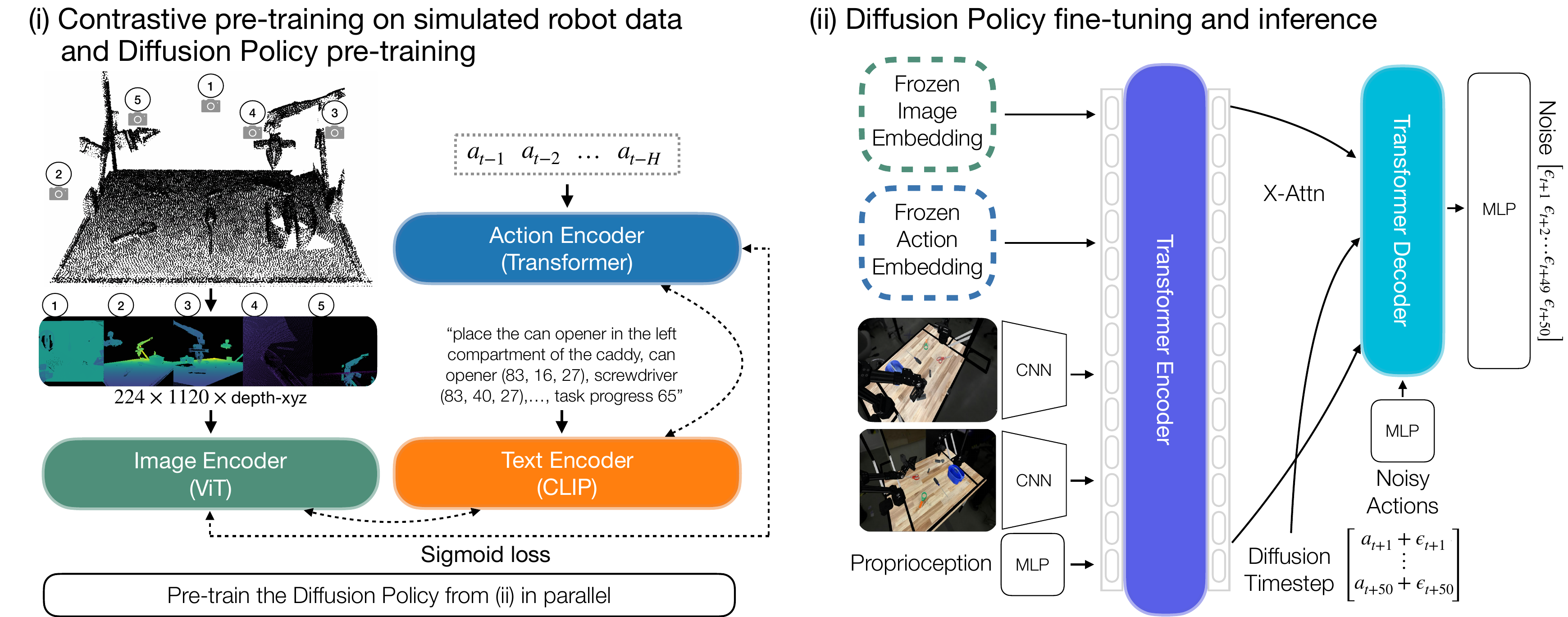}
    \caption{\textbf{Overview of \method}. \textbf{(i):} \method consists of three encoders: image, action and text. The image encoder is a Vision Transformer (ViT) \cite{dosovitskiy2020image} that takes five multi-view observations (overhead, back-right, front-left, wrist-left, and wrist-right) as input. These views are rendered from a merged point cloud and include depth and 3D coordinates. The action encoder is a Transformer encoder that takes a history of previous actions ($a_{t-1},\dots,a_{t-H}$) as input. The text input includes a general task description, each object's name and normalized position in the global coordinate frame, and an integer indicating task progress, and is processed by a CLIP \cite{radford2021learning} text encoder. The intuition is to learn image and action representations that capture correlations between action patterns and robot states as observed in images, grounded by text that describes the spatial and temporal information about the objects and task. We also pre-train a Diffusion Policy \cite{zhao2024aloha} in parallel. 
    \textbf{(ii):} The Diffusion Policy is first initialized with the pre-trained weights from stage (i) and use the frozen \method encoders to extract image and action embeddings.
    These embeddings are fed into a Transformer encoder together with RGB feature maps from ResNet-50 backbones and proprioceptive features from the robot to predict noise of shape $50 \times 14$ for the next 50 actions ($\epsilon_{t+1},\dots, \epsilon_{t+50}$).
    }
    \label{fig:model_architecture}
    % \vspace{-1.5em}
\end{figure*}

Recent works on robotic manipulation \cite{intelligence2025pi05visionlanguageactionmodelopenworld,nvidia2025gr00tn1openfoundation} have explored leveraging web-scale images, videos, and text with simulated and real-world robot data for pre-training. However, it is not well understood which data sources contribute most to policy fine-tuning performance. Moreover, these approaches do not pre-train on 3D data or incorporate 3D perception, which provides the spatial information robots need to plan and act effectively in the physical world. A recent work, 3D-MVP \cite{qian20253d}, proposes a 3D multi-view pre-training framework using masked autoencoders for manipulation. This method learns image representations through image reconstruction, but it does not leverage robot action data during pre-training. We hypothesize that 3D perception and robot actions are valuable pre-training modalities because they are fundamental to how robots perceive and act in the world.

In this work, we propose \textbf{C}ontrastive \textbf{L}earning for 3D Multi-View \textbf{A}ction-Conditioned Robotic \textbf{M}anipulation \textbf{P}retraining (\textbf{\method}). We contrastively pre-train an image encoder, a text encoder, and an action encoder on large-scale simulated robot data by pulling positive image–text, image–action, and text–action pairs closer together while pushing negative pairs apart. See \autoref{fig:hero_figure} for an overview.

For image representation, we re-render multi-view four-channel observations from the merged point cloud that include depth and 3D coordinates, using fixed viewpoints and dynamic wrist views, which provide clearer observations for high-precision manipulation tasks. We apply STRING \cite{schenck2025learning} relative positional encoding to the image encoder, enabling tokens from different views that correspond to nearby regions in 3D space to be correlated. For action data, we use a history of previous actions. Our intuition is that many manipulation tasks share common action patterns or trajectory segments that can be learned during pre-training, thereby reducing the amount of task-specific learning required during fine-tuning. For text, we extract object spatial relationships and task progress from the simulator and use this information to ground image and action representations.
Note that the text encoder is used exclusively during pre-training; therefore, \method does not rely on privileged state information during policy deployment.
In addition, we also pre-train a Diffusion Policy, the visuomotor controller for the robot, in parallel to initialize the policy for fine-tuning. We find that appropriate initialization substantially improves fine-tuning efficiency and performance. We then fine-tune the policy using task-specific demonstrations for each task.

In summary, our contributions are as follows: (i) We introduce a novel 3D contrastive pre-training framework for robotic manipulation that leverages robot actions and dynamic wrist views. (ii) We are the first to apply STRING relative positional encoding using 3D coordinates derived from point clouds. (iii) We show that dynamic wrist views are critical for strong results on high-precision tasks, and that pre-training both the encoders and the policy yields substantially better performance than pre-training the encoders alone. (iv) Extensive empirical results show that \method pre-training greatly improves sample efficiency and policy performance during fine-tuning, outperforming state-of-the-art baselines.

\section{Related Work}

Contrastive learning has seen significant interest in recent years \cite{hu2024comprehensive} with success on tasks such as view-invariant dense object features \cite{schmidt2016self,hou2021pri3d} and robotic manipulation \cite{graf2023learning,lee2025class}.
One major area of interest in contrastive learning has been training large-scale vision encoders.
Modern approaches such as CLIP \cite{radford2021learning}, ALIGN \cite{jia2021scaling}, COCA \cite{yu2022coca}, OpenCLIP \cite{cherti2023reproducible}, SigLIP \cite{zhai2023sigmoid}, Florence \cite{yuan2021florence}, InternVL \cite{chen2024internvl}, and EVA \cite{fang2023eva,fang2024eva,sun2023eva} use contrastive learning to create semantically meaningful image embeddings by aligning vision encoders with text embeddings derived from image captions.
TIPS \cite{maninis2024tips} takes this a step further by introducing additional loss terms to maintain spatial features as well as semantic features in the image embeddings.
ULIP \cite{xue2023ulip} and ULIP-2 \cite{xue2024ulip} lift contrastive learning into 3D by aligning a 3D encoder with pre-aligned text and image encoders.
In this paper, we utilize contrastive learning to pre-train 3D, text, and action encoders.

Interfacing unordered point clouds with learned models has been a challenge in computer vision and robotics.
PointNet \cite{qi2017pointnet} and PointNet++ \cite{qi2017pointnetb} showed early success by utilizing an MLP and max pooling layers for object classification and segmentation.
PointMLP \cite{ma2022rethinking} added to this a lightweight geometric affine module and residual connections.
DGCNN \cite{wang2019dynamic} takes a different approach, applying a convolutional kernel on a graph neural network constructed from the point cloud.
The advent of transformers \cite{vaswani2017attention}, operating on sets of tokens, created natural opportunities for new point-cloud model architectures.
PCT \cite{guo2021pct} and Point Transformer \cite{zhao2021point} used point-based features directly as the tokens for the self-attention mechanism.
However, these methods required significant downsampling of the point cloud to acheive tractability, even when utilizing local attention.
Point-BERT \cite{yu2022point} solved this by first pooling local regions of points together using a mini PointNet, and generating one token per region.
This methodology has been used and iterated on successfully by many follow-on works \cite{zhou2023uni3d,huang2024frozen,xu2024pointllm,huang2023embodied,zhang2022point,wu2022point,wu2024point} with ULIP \cite{xue2023ulip} and ULIP-2 \cite{xue2024ulip} evaluating both PointNet and Point-BERT as point-cloud encoder backbones.
An alternative approach to Point-BERT style architectures is the multi-view projection approach, i.e., projecting the point cloud to multiple virtual image planes and applying  standard 2D image architectures to the resulting images.
This has shown success for point-cloud text alignment \cite{zhang2022pointclip,zhu2023pointclip}, point-cloud captioning \cite{ma2025multi}, shape classification and recognition \cite{goyal2021revisiting,hamdi2021mvtn,hamdi2021voint}, and 3D visual scene grounding \cite{huang2022multi}.

Prior work in robotics has shown that these point-cloud encoders increase performance. DP3 \cite{ze20243d} showed improved success rate across 72 simulation tasks using a PointNet-style backbone.
ArticuBot \cite{wang2025articubot} and HACMan \cite{zhou2023hacman} utilized PointNet++ backbones to solve articulated object and non-prehensile manipulation tasks respectively.
PerAct \cite{shridhar2023perceiver} and GNFactor \cite{ze2023gnfactor} voxelize their point clouds and feed them into a Perceiver Transformer \cite{jaegle2021perceiver} showing high success rates across multiple tasks.
Act3D \cite{gervet2023act3d} generates point clouds from multiple RGB-D images, samples potential keypoints in a course-to-fine manner, and cross-attends them to select grasp points.
RVT \cite{goyal2023rvt} and 3D-MVP \cite{qian20253d} took the multi-view projection approach and used five orthogonal image planes (1 for each cardinal direction plus a top-down view) to maximize view coverage of the point cloud, yielding significant gains in success rate.
Lift3D \cite{jia2025lift3d} combined both the multi-view projection approach with a Point-BERT backbone by using projection to get features from 2D models, then averaging them over views in the point cloud and generating embeddings using Point-BERT.
In our work, we found the multi-view approach worked best (see Sec. \ref{sec:ablations}), with three fixed views (overhead, front-left, and back-right) and two dynamic views (a view attached to the wrist of each arm).

\section{Methodology}

\subsection{Problem Formulation}

In contrastive pre-training, we aim to learn a model $\psi$ that generates image and action embeddings that can be utilized to improve downstream fine-tuning of a behavior cloning policy. Given a mini-batch of image-text-action triplets $\mathcal{B} = \{(I_1, T_1,A_1),\dots,(I_n,T_n,A_n)\}$, the contrastive objective pulls the embeddings of matched triplets $(I_i, T_i,A_i)$ closer together while pushing apart embeddings of mismatched triplets $(I_i, T_j,A_k)$ for $i \neq j$ and $i \neq k$. We denote the per-pixel 4-channel image observation by $I \in \mathbb{R}^{H_I \times W_I \times 4}$, where $H_I$ and $W_I$ are the image height and width, and the four channels $(D, X, Y, Z)$ correspond to depth and 3D coordinates. Let $T$ be a text string containing (i) a high-level task description, (ii) each object's name and its normalized position in the global coordinate frame, and (iii) an integer indicating task progress. Note that privileged information, such as an object’s name and position, is obtained from the simulator and is used only during contrastive pre-training. $A$ denotes the history of previous actions from a dual-arm robot: $A_{t} = \{A_{t-1},\dots,A_{t-H}\}$, where $H$ is the action chunk size and missing terms are zero-padded if $t - H <0$.

In policy fine-tuning, the objective is to train a policy $\epsilon_\theta$, parameterized by $\theta$, that learns from a dataset of expert demonstrations ${\mathcal{D} = \{E_1, \ldots, E_M\}}$ where ${E_i = (I_1, I^{R}_1, A_1, \ba_{1}, \ldots, I_t, I^{R}_t, A_t, \ba_{t})}$ for a demonstration with $t$ time steps, leveraging pretrained model $\psi$. Here, $I^{R}$ denotes an RGB image observation $I^{R} \in \mathbb{R}^{H_R \times W_R \times 3}$ where $(H_R, W_R)$ may differ from $(H_I, W_I)$, and $\ba$ denotes the target joint positions for the robot. Note that both $I$ and $I^{R}$ can be obtained from a standard RGB-D camera.

\subsection{Contrastive Pre-training}

\textbf{Rendering.} Given RGB-D images captured from a set of cameras, we construct a scene point cloud using the known camera intrinsics and extrinsics.
Following \cite{ze20243d}, we omit RGB colors from the point cloud for better appearance generalization.
We apply 3D bounding-box cropping and voxel-grid downsampling to remove points outside the workspace and reduce point density. We then re-render the scene from five virtual viewpoints, each with depth and 3D coordinates: overhead, front-left, back-right, wrist-left, and wrist-right.
The virtual wrist cameras are defined from the left and right gripper poses, with a small offset for improved visibility.
These wrist views reduce occlusion during object manipulation. 
We empirically find that these dynamic views outperform standard orthogonal views and are particularly effective on high-precision tasks. 

\textbf{Model Architecture.} The encoders in \method consist of an image encoder $f(\cdot)$, a text encoder $g(\cdot)$, and an action encoder $h(\cdot)$. For the image encoder, we use a Vision Transformer (ViT) \cite{dosovitskiy2020image} with STRING \cite{schenck2025learning} relative positional encoding. Specifically, we use a 94M-parameter, 12-layer, 768-wide, B/16 model with 12 attention heads and multihead attention pooling (MAP) \cite{lee2019set}.
For STRING, we use the Cayley STRING formulation, where the 3D coordinates for relative positional encodings are computed as the average global XYZ of the point-cloud points within each image patch.
Because the coordinates are in the global coordinate frame, attention between image patches from different images can be seamlessly attenuated by the STRING mechanism without any additional modification.
More analysis is provided in Sec. \ref{sec:theory}.

The text encoder takes a tokenized text \cite{raffel2020exploring,xue2021mt5} with 256 tokens as input, and we train a CLIP text encoder \cite{radford2021learning} on these tokens. 
We use a 277M-parameter, 12-layer, 768-wide model with 12 attention heads.
The action encoder is a standard Transformer encoder with MAP pooling. We use an 18M-parameter, 2-layer, 768-wide model with 12 attention heads. Each encoder outputs a 768-dimensional embedding.

\textbf{Contrastive Loss.} 
The encoders are trained to minimize a SigLIP-style loss \cite{zhai2023sigmoid} over the three modality pairs: image-text, image-action, and text-action.
\begin{equation}
\mathcal{L}_{\text{ImageText}} = -\frac{1}{|\mathcal{B}|}\sum_{i=1}^{|\mathcal{B}|}\sum_{j=1}^{|\mathcal{B}|}
\log\!\left(\frac{1}{1+e^{\,\ell_{ij}\left(-t\,\mathbf{x}_i^\top \mathbf{y}_j + b\right)})}\right)
\end{equation}

\begin{equation}
\mathcal{L}_{\text{ImageAction}} = -\frac{1}{|\mathcal{B}|}\sum_{i=1}^{|\mathcal{B}|}\sum_{k=1}^{|\mathcal{B}|}
\log\!\left(\frac{1}{1+e^{\,\ell_{ik}\left(-t\,\mathbf{x}_i^\top \mathbf{z}_k + b\right)})}\right)
\end{equation}

\begin{equation}
\mathcal{L}_{\text{TextAction}} = -\frac{1}{|\mathcal{B}|}\sum_{j=1}^{|\mathcal{B}|}\sum_{k=1}^{|\mathcal{B}|}
\log\!\left(\frac{1}{1+e^{\,\ell_{jk}\left(-t\,\mathbf{y}_j^\top \mathbf{z}_k + b\right)})}\right)
\end{equation}
where $\mathbf{x}_i = \frac{f(I_i)}{\left\lVert f(I_i) \right\rVert_2}$, $\mathbf{y}_i = \frac{g(T_i)}{\left\lVert g(T_i) \right\rVert_2}$, and $\mathbf{z}_i = \frac{h(A_i)}{\left\lVert h(A_i) \right\rVert_2}$. $\ell_{ij}$, $\ell_{ik}$, $\ell_{jk}$ are the labels for image-text, image-action, and text-action inputs, respectively; they are set to $1$ for matched pairs and to $-1$ otherwise. Bias $b$ and temperature $t$ are learnable scaling parameters that help stabilize training and prevent overly large updates early on. Following \cite{zhai2023sigmoid}, we initialize $b=-10$ and $t=\exp(\log 10)=10$. 
The total loss of \method encoders is $\frac{1}{3}(\mathcal{L}_{\text{ImageText}} + \mathcal{L}_{\text{ImageAction}} + \mathcal{L}_{\text{TextAction}})$. This formulation encourages paired modalities to be close while pushing unpaired modalities apart in a shared embedding space. The intuition is to learn action patterns or representations that can be used to infer future states, while also learning image representations that can be used to infer the history of preceding robot actions, conditioned on text describing the task type, the spatial arrangement of objects in the scene, and the temporal progress of the task.

\begin{figure*}[t]
    \centering
    \includegraphics[width=0.75\textwidth]{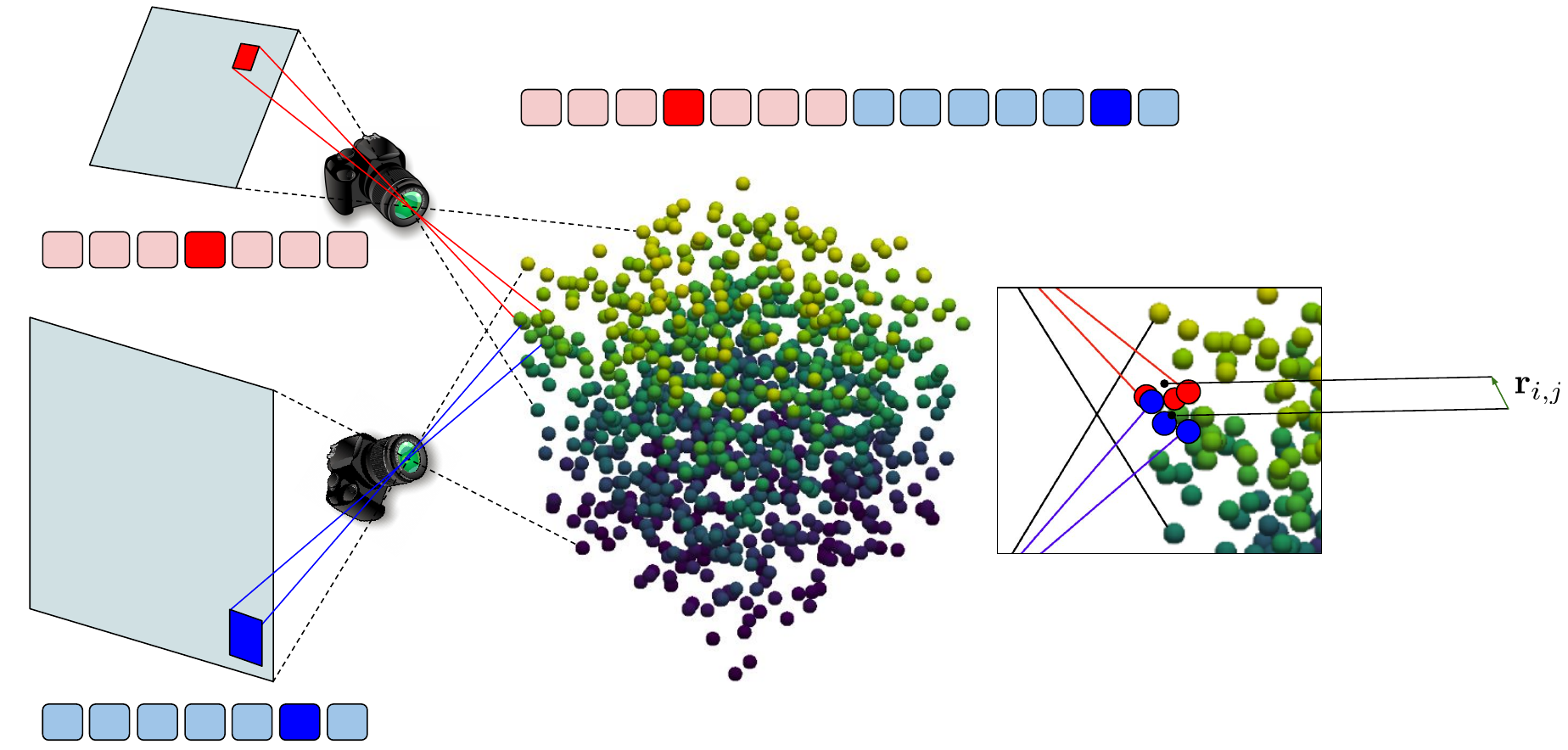}
    \caption{Pictorial desciption of the STRING mechanism applied in CLAMP's image encoder to correlate tokens form different views corresponding to similar regions of the 3D space. Tokens from different 2D views are concatenated, but also equipped with STRING positional encoding mechanism \citep{schenck2025learning}. Consequently, the attention score between two highlighted tokens from two different views will be modulated by a function that depends only on the vector $\mathbf{r}_{i,j} \in \mathbb{R}^{3}$ between the centers of mass of two sub-point clouds (highlighted in red and blue correspondingly) of the visible sets of points related to those tokens.}
    \label{fig:string-clamp}
    % \vspace{-1.5em}
\end{figure*}

\subsection{Diffusion Policy Pre-training and Fine-Tuning}

\textbf{Diffusion Policy Architecture.} The Diffusion Policy is based on the architecture proposed by Zhao et al. \cite{zhao2024aloha}. It has been shown to excel at handling multimodal action distributions \cite{chi2023diffusionpolicy,zhao2024aloha}, which is essential for challenging tasks with multiple action modes (e.g., grasp-to-place vs. grasp-to-hand-over-to-place) or for learning from demonstrations collected by different human operators. We use the Denoising Diffusion Probabilistic Model (DDPM) formulation \cite{ho2020denoising,chi2023diffusionpolicy}. During training, a denoising step $k$ is randomly selected. Conditioned on $k$, we randomly sample random noise $\epsilon^k$ with corresponding variance. The training objective of the noise prediction network $\epsilon_\theta$ is to predict the injected noise $\epsilon^k$: $\mathcal{L} = \mathrm{MSE}\Big(\epsilon^k,\ \epsilon_\theta(\mathbf{O}_t,\ \ba_t^0 + \epsilon^k,\ k)\Big)$.  
At test time, the policy samples an action by iteratively denoising an action sequence:
\begin{equation}
\ba_t^{k-1} = \alpha\Big(\ba_t^k - \gamma\, \epsilon_\theta(\mathbf{O}_t,\ba_t^k,k) + \mathcal{N}(0,\sigma^2 I)\Big),
\end{equation}
where $\ba_t^k$ is sampled from Gaussian noise; $\alpha$, $\gamma$, and $\sigma$ are determined by the DDPM noise schedule; $\mathcal{N}(0,\sigma^2 I)$ denotes additive Gaussian noise; and $\mathbf{O}_t$ is the observation at time t.

We use a 192M-parameter Diffusion Policy with a 4-layer, 768-wide Transformer encoder with 8 attention heads and a 7-layer, 768-wide Transformer decoder with 8 attention heads. Following \cite{zhao2024aloha}, we apply action chunking with a chunk size of 50, and the policy outputs 14 absolute target joint positions. For each camera’s RGB image, we use a separate ResNet-50 vision backbone initialized with weights pretrained on ImageNet classification \cite{5206848}. We extract feature maps from the stage-four output of ResNet. These feature maps are concatenated with unpooled frozen image and action embeddings from pretrained \method encoders, taken before MAP pooling, along with features from the robot’s proprioception. We add positional embeddings before feeding the resulting tokens into the encoder. The encoder latents are then passed to the decoder together with a noised action chunk, a learnable positional embedding, and an one-hot encoding of the diffusion timestep. Finally, the decoder output is passed through a projection head to predict noise with shape $50 \times 14$ for the next 50 actions.

\textbf{Pre-training and Fine-Tuning.} For pre-training, we train the policy on the same dataset used to pre-train the \method encoders for simulated tasks. For real-world tasks, we pre-train the policy on real-world training data, since weights learned from real data provide better initialization due to reduced sim-to-real gap compared with policies pre-trained on simulated data. Note that frozen image and action embeddings are not used during policy pre-training because the policy is trained in parallel with the encoders. For fine-tuning, we initialize all policy weights, including each ResNet backbone, from the pre-trained checkpoint. This pre-training initialization is essential for achieving high fine-tuning sample efficiency and performance, as shown in \autoref{tab:sota_results}. Our key insight is to pre-train both the encoders and the policy before fine-tuning the policy. See \autoref{fig:model_architecture} for an overview of the model architecture for pre-training and fine-tuning, and refer to Appendix \ref{sec:implementation_details} for implementation details.

\subsection{Theoretical Analysis}
\label{sec:theory}
CLAMP is the first model that applies STRING \citep{schenck2025learning} relative positional encoding (RPE) mechanism with coordinate vectors coming directly from the $(x,y,z)$ coordinates of points from the point cloud (rather than by leveraging depth, as it was the case in the 3D-STRING variants from \citep{schenck2025learning}).

\paragraph{Multi-view setting with STRING} CLAMP solves in an elegant way the problem of modeling the correlations between tokens coming from different 2D-views. Rather than leveraging non-trivial methods for finding matchings between patches from different views, but corresponding to the same or similar regions of 3D space, CLAMP does it implicitly via STRING positional encoding, associating with every token its coordinate vector, pointing to the center of mass of the set of visible points corresponding to the token. Those coordinate vectors are then used in STRING for positional encoding (see: Fig. \ref{fig:string-clamp}). Below, we explain the mechanism in greater depth.

\begin{lemma}
\label{lemma:cayley-lemma}
Consider two tokens: $t_{\mathcal{P}_{1}}$ and $t_{\mathcal{P}_{2}}$ of the image encoder's self-attention module used in CLAMP, corresponding to visible sub-point clouds $\mathcal{P}_{1}$ and $\mathcal{P}_{2}$ respectively. Denote by $(\mathbf{q}_{\mathcal{P}_{1}}$, $\mathbf{k}_{\mathcal{P}_{1}}) \in \mathbb{R}^{d_{QK}} \times \mathbb{R}^{d_{QK}}$ and $(\mathbf{q}_{\mathcal{P}_{2}}$, $\mathbf{k}_{\mathcal{P}_{2}}) \in \mathbb{R}^{d_{QK}} \times \mathbb{R}^{d_{QK}}$ their corresponding query-key vector-pairs and by $\mathbf{r}_{i,j} \in \mathbb{R}^{3}$ a vector joining the center of mass of $\mathcal{P}_{i}$ with the center of mass of $\mathcal{P}_{j}$ for $i,j \in \{1, 2\}$. For a given vector $\mathbf{u} \in \mathbb{R}^{d_{QK}}$, denote by $\mathbf{u}(\mathcal{C})$ its representation in the coordinate system $\mathcal{C}$ (we assume that the canonical system $\mathcal{C}(\mathbf{e}_{1},...,\mathbf{e}_{d_{QK}})$ is the default one, thus: $\mathbf{u}(\mathcal{C}(\mathbf{e}_{1},...,\mathbf{e}_{d_{QK}}))=\mathbf{u}$). Then the not-normalized logit score $s_{i,j}$ of the self-attention module in CLAMP's image encoder is given as follows, for $\alpha(\mathbf{v},\mathbf{w}) \in [0, \pi]$ denoting an angle between vectors $\mathbf{v},\mathbf{w}$, learnable coordinate system $\mathcal{L}$ and linear functions $\delta_{k}$ ($k=1,...,l=\frac{d_{QK}}{2}$; without loss of generality, we assume that query/key dim is even):
\begin{align}
\begin{split}
\label{eq:s_equation}
s_{i,j} = \sum_{k=0}^{l-1} \|\mathbf{q}_{\mathcal{P}_{i}}(\mathcal{L})[2k,2k+1]\|_{2}\|\mathbf{k}_{\mathcal{P}_{j}}(\mathcal{L})[2k,2k+1]\|_{2} \cdot \\
\cdot \cos\left(\alpha(\mathbf{q}_{\mathcal{P}_{i}}(\mathcal{L})[2k,2k+1], \mathbf{k}_{\mathcal{P}_{j}}(\mathcal{L})[2k,2k+1]) + \delta_{k}(\mathbf{r}_{i,j})\right)    
\end{split}
\end{align}
\end{lemma}
Lemma \ref{lemma:cayley-lemma} sheds a light on the mechanism used in CLAMP to correlate tokens corresponding to sub-point clouds within different distances from each other (proof in Appendix \ref{cayley-proof}).

\textbf{Tokens encoding the same sub-point cloud:} In that scenario, $\mathbf{r}_{i,j}=0$ and thus $\delta_{k}(\mathbf{r}_{i,j})=0$ (since $\delta_{k}$ is a linear function). Since coordinate system change is an isometry, by Lemma \ref{lemma:cayley-lemma}, it is easy to see that the formula for $s_{i,j}$ reduces to: $s_{i,j} = \mathbf{q}_{\mathcal{P}_{i}}^{\top}\mathbf{k}_{\mathcal{P}_{j}}$. Thus we obtain a regular attention score.

It turns out that STRING method applied in CLAMP has the ability to modulate large attention values of the regular attention mechanism in such a way that larger distances $\|\mathbf{r}_{i,j}\|_{2}$ between centers of the sub-point clouds imply smaller modulated attention scores. Intuitively, this is a very useful property, since the mechanism is capable of implicitly learning the 3D structure of the input, correlating tokens from different views, but corresponding to the same or similar regions of space (by explicitly reducing attention scores for sub-point clouds within larger distances from each other). The following is true:

\begin{lemma}
\label{lemma:decreasing}
Take some $T>0$ and $\phi \in [0, \frac{\pi}{3}]$. Consider two sub-point clouds $\mathcal{P}_{i}$ and $\mathcal{P}_{j}$ with the corresponding vector $\mathbf{r}_{i,j}$ joining their centers of masses. Take a \textit{translated} version $\widehat{\mathcal{P}}_{j}(\mathbf{r})$ of $\mathcal{P}_{j}$ such that the vector joining the centers of masses becomes $\mathbf{r}=\eta\mathbf{r}_{i,j}$ for some $\eta \in \mathbb{R}$. Assume that the query $\mathbf{q}_{\mathcal{P}_{i}} \in \mathbb{R}^{d_{QK}}$ and key $\mathbf{k}_{\mathcal{P}_{j}} \in \mathbb{R}^{d_{QK}}$ are taken from the ball $\mathcal{B}(0, R)$ centered at zero and of radius $R$ and that the absolute values of all their dimensions are lower-bounded by some $r>0$. Then there exists $\epsilon(R, r, \phi)>0$ and a linear function $\delta=\delta(\mathbf{r})$ such that the logit score $s_{i,j}(\mathbf{r})$ corresponding to $\mathcal{P}_{i}$ and $\widehat{\mathcal{P}}_{j}(\mathbf{r})$ is a decreasing function of $\|\mathbf{r}\|_{2}$ for $|\eta| \in [2T \phi, T(\pi-\phi)]$ if the cosine similarity between $\mathbf{q}_{\mathcal{P}_{i}}$ and $\mathbf{k}_{\mathcal{P}_{j}}$ is $ \geq 1-\epsilon$.
\end{lemma}

\begin{proof}
Consider a function: $f: \mathbb{R}^{d_{QK}} \times \mathbb{R}^{d_{QK}} \rightarrow \mathbb{R}$ defined as follows:   
\begin{equation}
f(\mathbf{x},\mathbf{y}) = \max_{k=0}^{\frac{d_{QK}}{2}-1} 
\arccos(\frac{\mathbf{x}[2k,2k+1]^{\top}\mathbf{y}[2k, 2k+1]}{\|\mathbf{x}[2k, 2k+1]\|_{2}\|\mathbf{y}[2k, 2k+1]\|_{2}})
\end{equation}
Denote by $L(r)>0$ the Lipschitz constant of the function $f$ in the region 
$\{(\mathbf{x},\mathbf{y}):\mathbf{x},\mathbf{y} \in \mathcal{B}(0, R), |\mathbf{x}[i]|,|\mathbf{y}[i]| \geq r; i=0,1,...,d_{QK}-1\}$.
Note that since $r>0$, $L(r)$ is well defined (i.e. is finite). Straightforward trigonometric analysis then leads to the conclusion that if $\epsilon>0$ is chosen to be small enough such that: $L(r) \cdot 2R\sin(\frac{\arccos(1-\epsilon)}{2}) \leq \phi$ then the following holds:
each angle $\alpha_{k} \in [0, \pi]$ between $\mathbf{q}_{\mathcal{P}_{i}}[2k,2k+1]$ and $\mathbf{k}_{\mathcal{P}_{j}}[2k,2k+1]$ for $k=0,1,...,\frac{d_{QK}}{2}$ is upper-bounded by $\phi$. Define $\delta(\mathbf{r})$ as follows: $\delta(\mathbf{r}) = \frac{1}{T}\frac{\mathbf{r}^{\top}\mathbf{r}_{i,j}}{\|\mathbf{r}_{i,j}\|^{2}_{2}}$ (in other words: $\delta(\mathbf{r})=\frac{1}{T}\eta$). But then for $|\eta| \in [2T\phi,T(\pi-\phi)]$, the $\cos$-expression from Eq. \ref{eq:s_equation} is a decreasing function of $\|\mathbf{r}\|_{2}$ (note that the vanishing $\cos$-expression corresponds to the regular logit score $s_{i,j}=\mathbf{q}_{\mathcal{P}_{i}}^{\top}\mathbf{k}_{\mathcal{P}_{j}}$, as discussed above). That completes the proof.
\end{proof}

Lemma \ref{lemma:decreasing} quantifies the intuition that high attention scores between tokens for which the corresponding sub-point clouds are close do not change much when the STRING-based modulation is applied, whereas those for pairs of tokens with larger distances can significantly decrease. 
Lemma \ref{lemma:decreasing} does it in the curated setting with one query-key pair and varying distances via the translation operation, whereas in the most general setting, we have many pairs and the analysis is more convoluted. Note however that the fact that angles $\alpha_{k}$ can correspond to the learnable (rather than canonical, as in the proof of Lemma \ref{lemma:decreasing}) coordinate system $\mathcal{L}$ (see: Lemma \ref{lemma:cayley-lemma}) makes $\delta$ learning easier. Even random $\mathcal{L}$ are useful, since they tend to distribute the norm of the queries/keys uniformly across different dimensions. This leads to the less constrained statements (for instance, the assumption about the lower bound on the absolute values of query/key dimensions is not needed).

\begin{table*}[t]
\centering
\caption{Results from simulation experiments comparing \method against five baselines, averaged over three training seeds. The `\method Encoders Only' variants pre-train only the encoders, while the policy is fine-tuned from scratch.}
\label{tab:sota_results}
\setlength{\tabcolsep}{6pt} %
\begin{tabular}{lcccccc}
\toprule
\multirow{2}{*}{\textbf{Method}} & \textbf{Can Opener} & \textbf{Screwdriver} & \textbf{Pen in} & \textbf{Mug on} & \textbf{Plate on} & \textbf{Drawer} \\
 & \textbf{in Caddy} & \textbf{in Caddy} & \textbf{Container} & \textbf{Plate} & \textbf{Rack} & \textbf{Open} \\
\midrule
VLA Backbone & 7.3\% & 38.0\% & 2.0\% & 14.7\% & 56.7\% & 33.3\% \\
ALOHA Unleashed & 68.0\% & 78.0\% & 16.0\% & 54.0\% & 76.0\% & 80.7\% \\
Pretrained ALOHA Unleashed & 70.0\% & 86.7\% & 20.7\% & 62.0\% & 76.7\% & 88.7\% \\
ACT & 34.0\% & 82.0\% & 9.3\% & 38.7\% & 70.0\% & 44.7\% \\
Pretrained ACT & 7.3\% & 20.0\% & 0.7\% & 6.0\% & 29.3\% & 16.7\% \\
\midrule
ACT with \method Encoders Only & 72.7\% & 92.0\% & 10.7\% & 62.7\% & 88.7\% & 28.0\% \\
ALOHA Unleashed with \method Encoders Only & 71.3\% & 83.3\% & 34.0\% & 86.0\% & 88.7\% & 82.7\% \\
ALOHA Unleashed with \method & \textbf{94.0\%} & \textbf{98.0\%} & \textbf{51.3\%} & \textbf{95.3\%} & \textbf{92.7\%} & \textbf{95.3\%} \\
\bottomrule
\end{tabular}
% \vspace{-1.5em}
\end{table*}

\begin{figure*}[t] %
    \centering
    
    \subfloat{
        \includegraphics[width=0.3\linewidth]{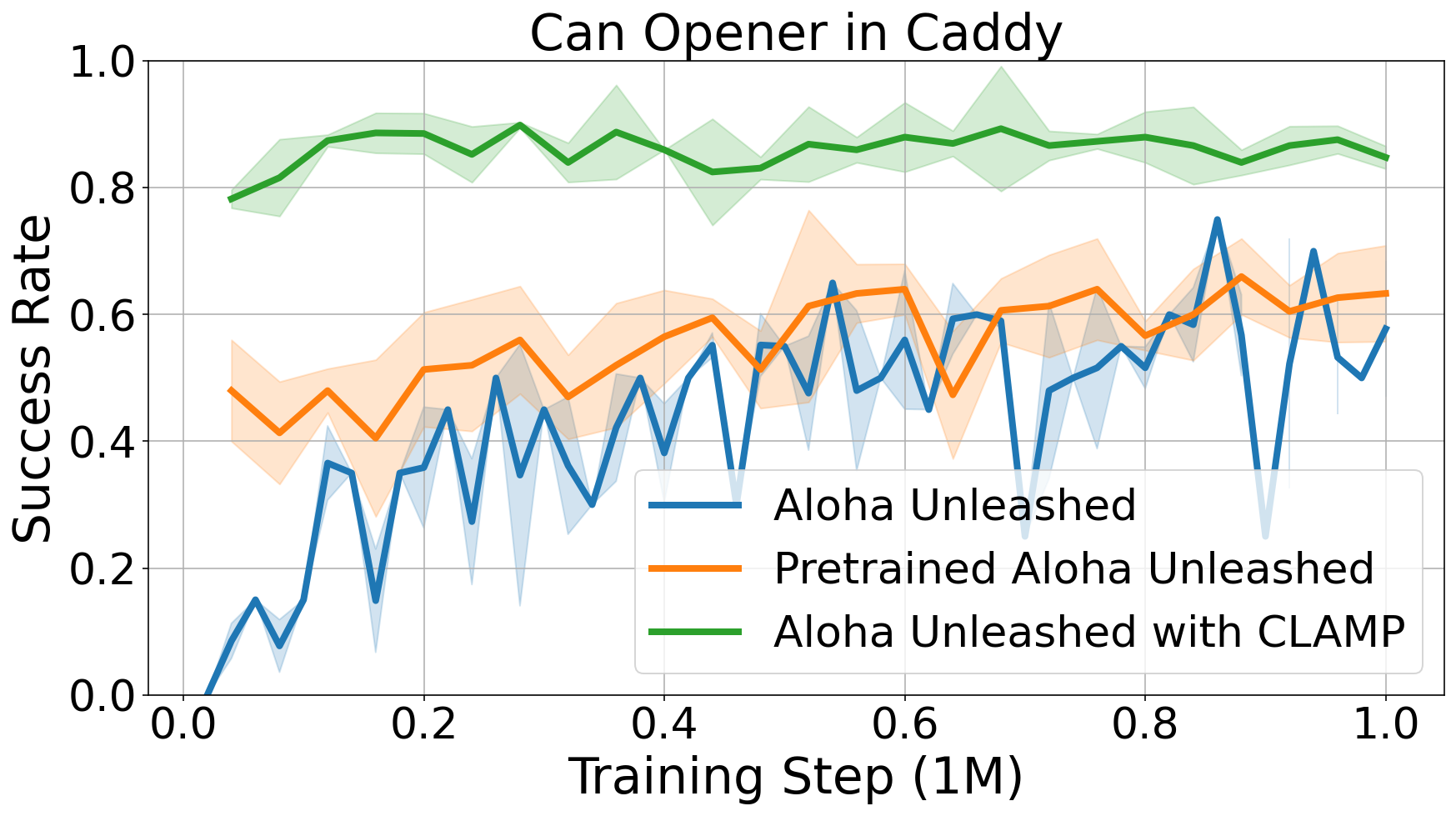}
    }
    \hfill
    \subfloat{
        \includegraphics[width=0.3\linewidth]{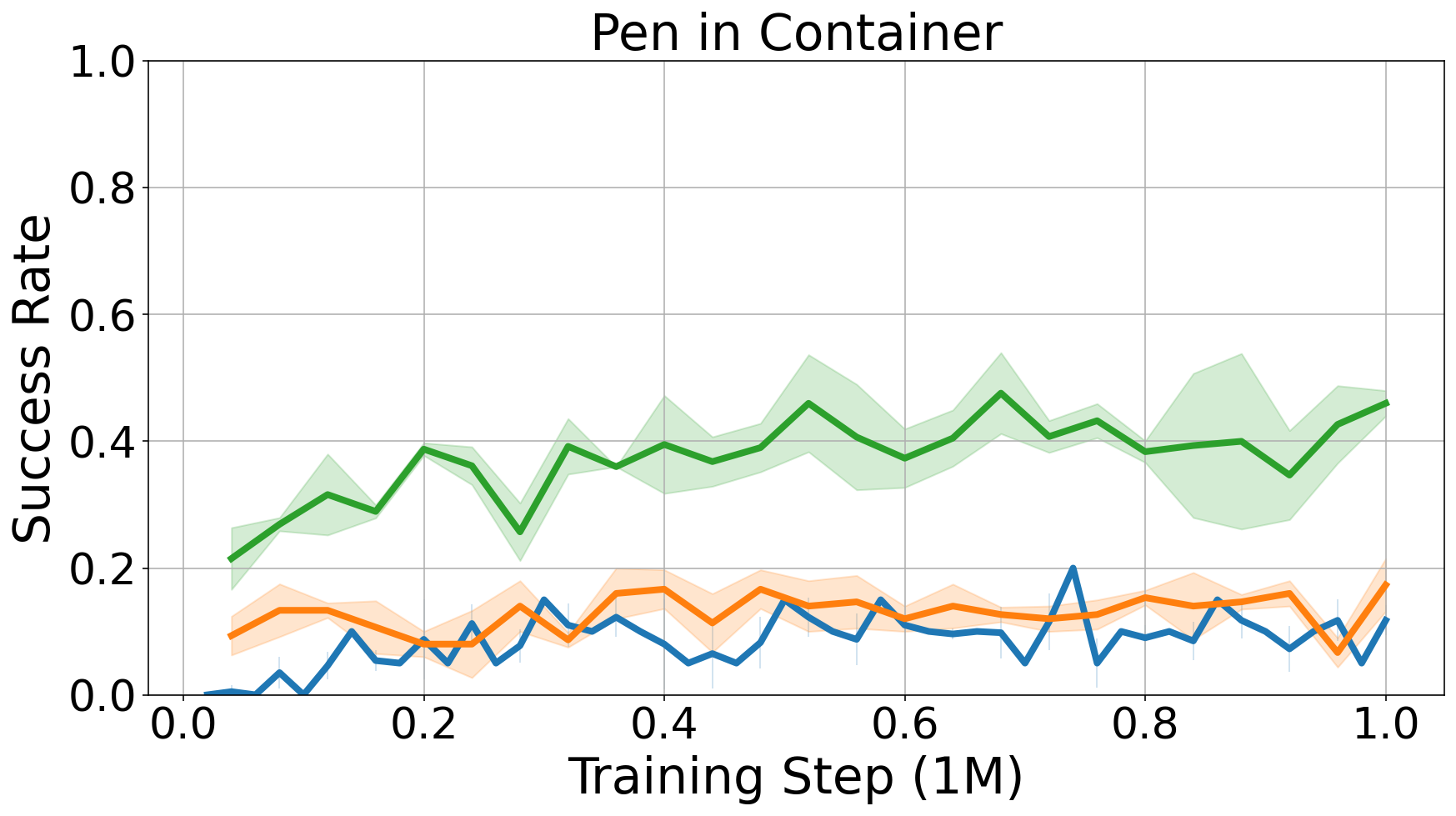}
    }
    \hfill
    \subfloat{
        \includegraphics[width=0.3\linewidth]{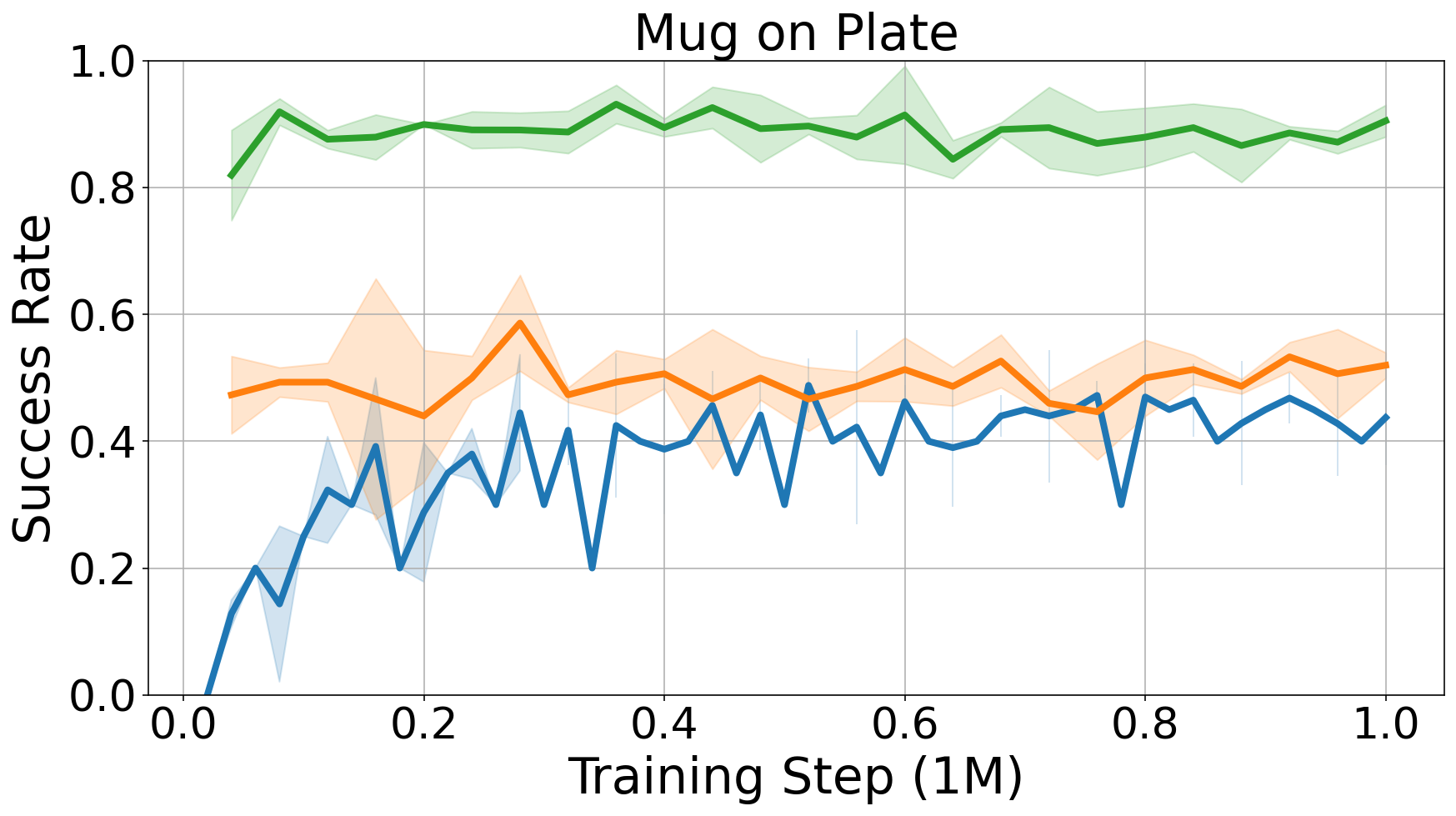}
    }
    
    \subfloat{
        \includegraphics[width=0.3\linewidth]{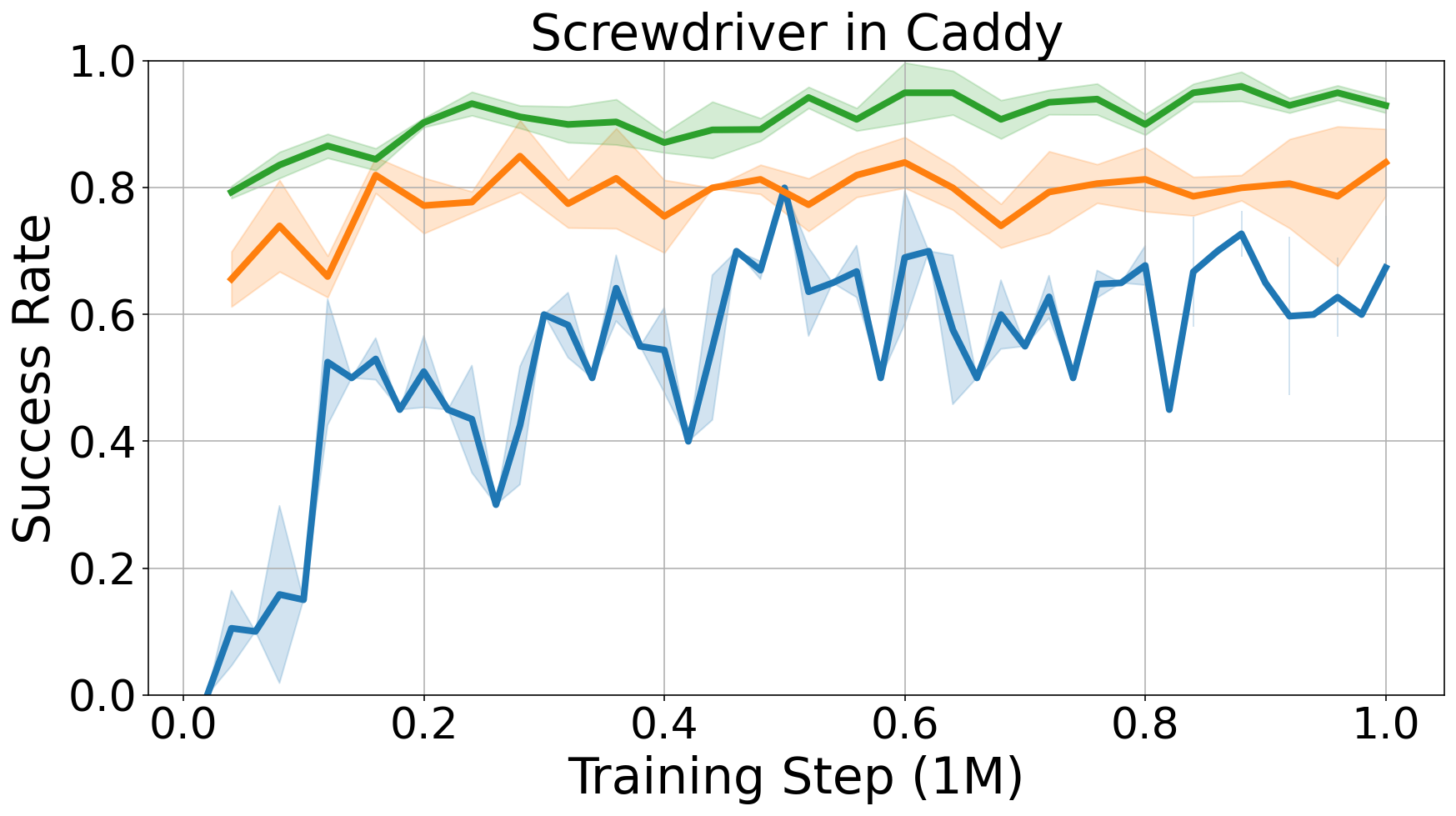}
    }
    \hfill
    \subfloat{
        \includegraphics[width=0.3\linewidth]{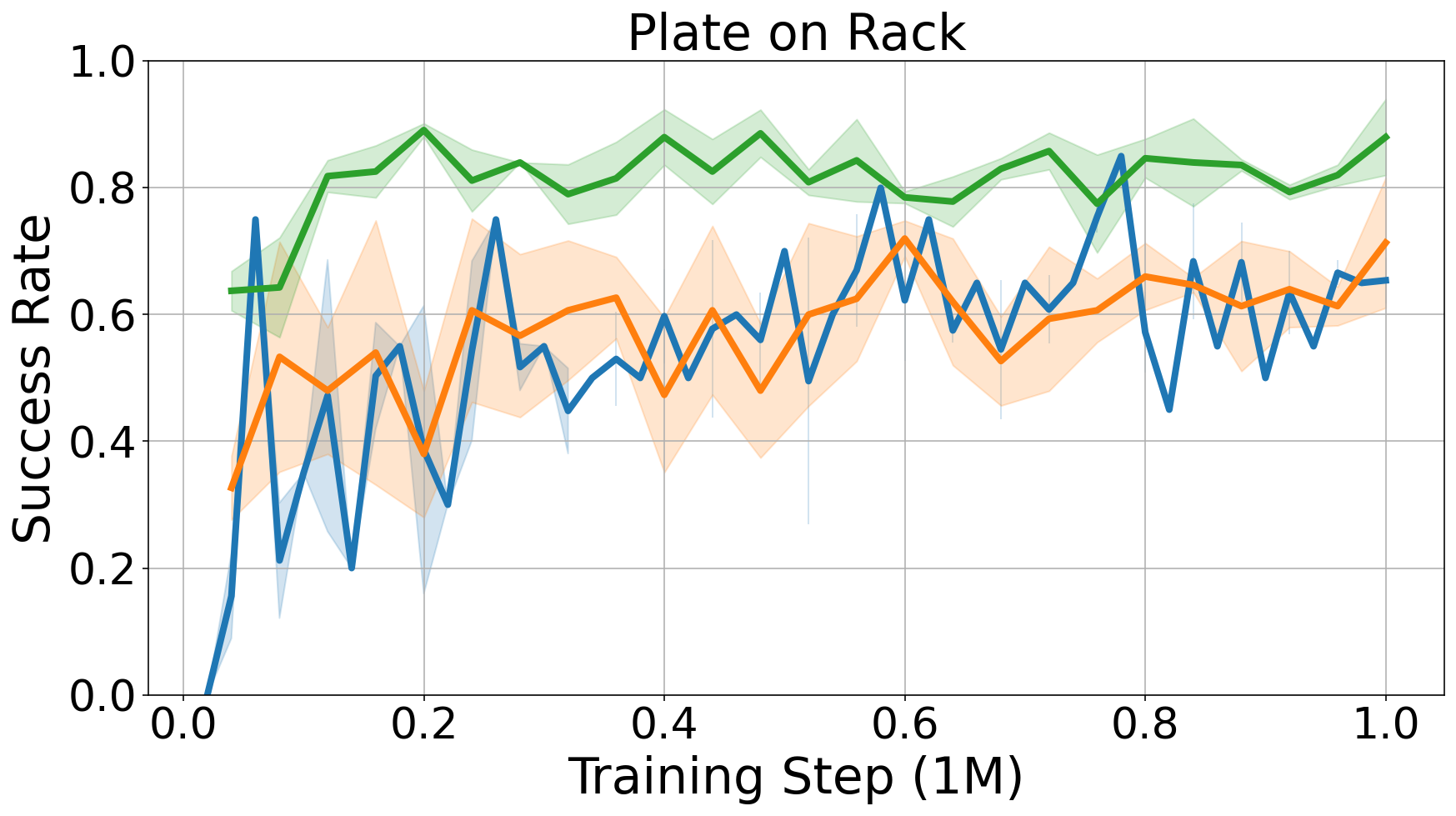}
    }
    \hfill
    \subfloat{
        \includegraphics[width=0.3\linewidth]{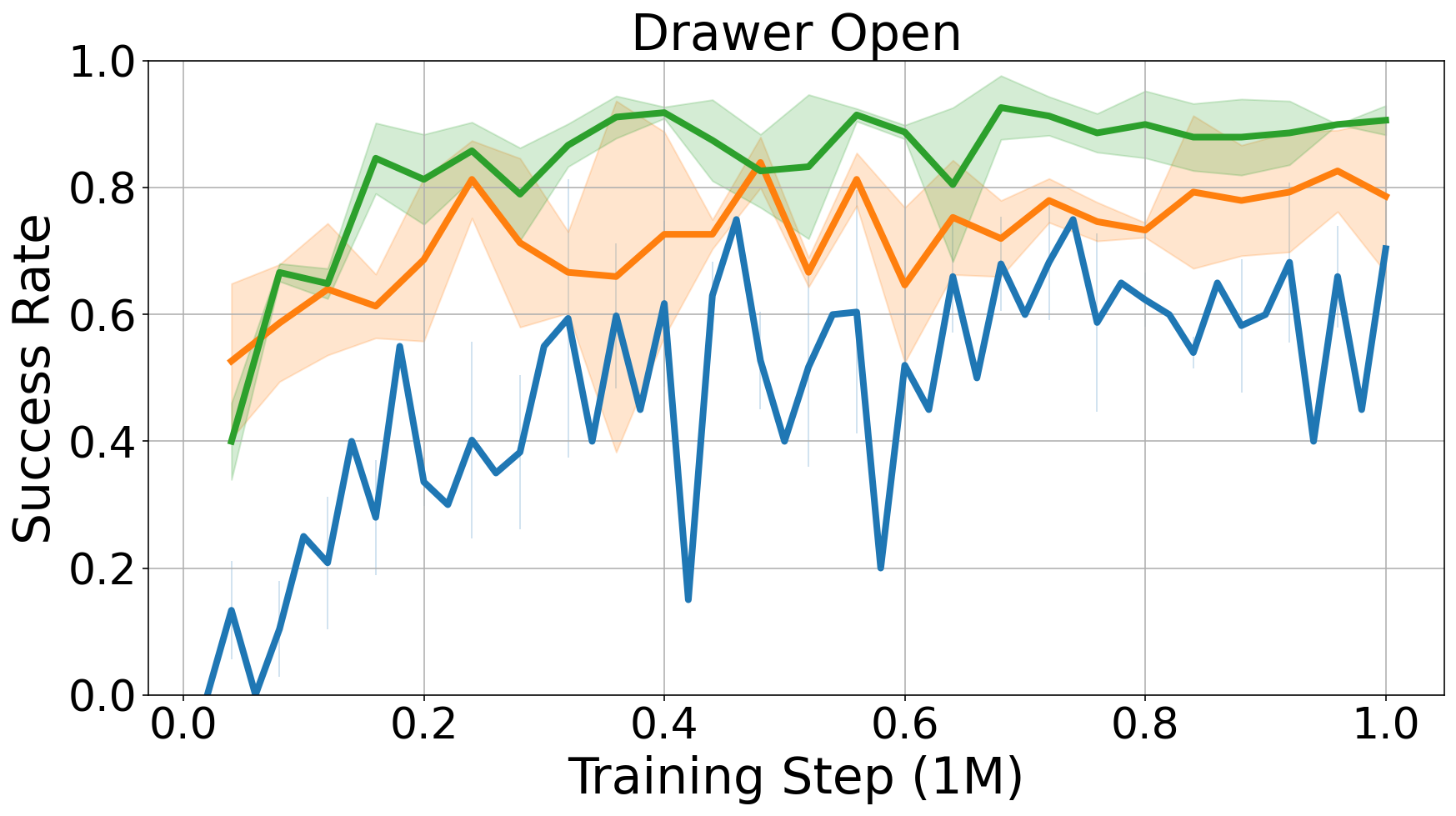}
    }
    
    \caption{Evaluation curves comparing our method (green) to baselines across three seeds in simulation. All methods are trained on 4500 demonstrations for 1M environment steps and evaluated for 50 trials per checkpoint every 40K steps.}
    \label{fig:eval_curves}
    % \vspace{-1.5em}
\end{figure*}

\section{Evaluation}

We utilized the ALOHA 2 model from Mujoco Menagerie \cite{todorov2012mujoco, zakka2022mujoco} for our simulation environment, the same model used in ALOHA Unleashed \cite{zhao2024aloha}.
This is a bimanual tabletop environment with two arms attached on either side of a table.
In simulation, we use a two-camera setup with front-left and back-right views. These cameras are placed on opposite corners of the table, 70 cm above the surface, and pointed at the center of the table. In the real world, we use a four-camera setup consisting of the two views used in simulation and a wrist-mounted camera on each robot arm. All cameras record RGB-D images.
The simulation and real robots were identical.
We follow the same simulation evaluation protocol as in \cite{zhao2024aloha}: we perform 50 evaluations per checkpoint every 40k training steps and report the maximum evaluation score across all checkpoints.
We also report the maximum success rate across all checkpoints for real-world experiments, and we evaluated at most two checkpoints per method per task.

\subsection{Evaluation Tasks}

\textbf{Simulation.} \texttt{Can Opener in Caddy}: Six objects are placed on the table: a caddy, screwdriver, magnifier, can opener, thumb drive, and scissors. They are evenly spaced in the $2 \times 3$ grid. The caddy, screwdriver, and magnifier may appear in any of the three top positions, while the can opener, thumb drive, and scissors may appear in any of the three bottom positions. Objects are randomly rotated by $\pm$90 degrees. The robot must grasp the can opener and place it in the left caddy.

\texttt{Screwdriver in Caddy}: Same setup as \texttt{Can Opener in Caddy}, but objects are randomly rotated by $\pm$180 degrees. The robot must grasp the screwdriver and place it in the right caddy.

\texttt{Pen in Container}: Six objects are placed on the table: a plate, bowl, container, mug, pen, and banana. They are also evenly spaced in a $2 \times 3$ grid. The plate, bowl, and container may appear in any of the three top positions, while the mug, pen, and banana may appear in any of the three bottom positions. Objects are randomly rotated by $\pm$180 degrees. The robot must grasp the pen and place it in the container.

\texttt{Mug on Plate}: Same setup as \texttt{Pen in Container}. The robot must grasp the mug and place it on the plate.

\texttt{Plate on Rack}: Two objects are placed on the table: a plate and drying rack. The drying rack is placed in the top half of the workspace, and the plate in the bottom half of the workspace. The drying rack’s z-axis orientation is randomly varied by $\pm$86.4 degrees. The robot must grasp the plate and place it into one of the dividers in the drying rack.

\texttt{Drawer Open}: A drawer is randomly placed in the workspace and rotated by $\pm$45 degrees. The robot must open the bottom drawer.

\textbf{Real-World.} \texttt{Open Drawers}: a drawer is randomly placed in the workspace and rotated by $\pm$90 degrees. The robot must open both the top and bottom drawers, with partial success awarded if the bottom drawer is opened.

\texttt{Close Drawers}: Same setup as \texttt{Open Drawers}, but with the drawers already open. The robot must close both the top and bottom drawers, with partial success awarded if the top drawer is closed.

\texttt{Mug on Plate}: Same setup as the \texttt{Mug on Plate} task in simulation. The robot must grasp the mug and place it on the plate, with partial success awarded for a successful grasp of the mug.

\texttt{Plate on Rack}: Same setup as the \texttt{Plate on Rack} task in simulation, but the drying rack’s z-axis orientation is randomly varied by $\pm$45 degrees. The robot must grasp the plate and place it into one of the dividers in the drying rack, with partial success awarded for a successful grasp of the plate.

\texttt{Recycle Cans}: Two trash bins are placed at the top-right of the workspace, and three soda cans are placed in the center. The cans are randomly rotated by $\pm$45 degrees. The robot must grasp at least two cans and place them in the bins, with partial success awarded for grasping and placing one can.

See \autoref{fig:real_world_environments} for examples of the initial configurations for each real-world task.

\subsection{Datasets}

The datasets are collected using a combination of a VLA policy in simulation and teleoperation in real, as described in the following sections.

\textbf{Sim Data.} The sim data was collected using a pre-trained VLA (vision-language-action) model. 
Specifically, we used 
Gemini Robotics 1.5 (GR1.5) \cite{gr15},
which takes camera images, a text command, and robot proprioception as inputs and outputs target joint positions.
It generated data across 48 tasks from which we only kept episodes that successfully completed the task.
It averaged 11,533 successful episodes per task for a total of 553,592 episodes to be used for pre-training our encoders (though not uniformly distributed across tasks, see Appendix \ref{appendix:sim_dataset}).
Additionally we generated 4,500 episodes for 6 of the tasks to be used for fine-tuning the policies.

In total, 358,510,758 image–text-action triplets were used for pre-training, while each method was fine-tuned on a single task using 4,500 demonstrations. Note that the pre-training tasks do not overlap with the fine-tuning tasks. See Appendix \ref{appendix:sim_dataset} for a list of pre-training tasks and count of successful episodes per task.

\textbf{Real Data.} Experienced operators used the teleoperation system \cite{zhao2024aloha} to collect demonstrations: 500 each for \texttt{Open Drawers} and \texttt{Close Drawers}, 100 each for \texttt{Mug on Plate} and \texttt{Plate on Rack}, and 63 for \texttt{Recycle Cans}.

\subsection{Baselines}

We compare \method against strong baseline methods: \textbf{VLA Backbone}, \textbf{ALOHA Unleashed} \cite{zhao2024aloha}, \textbf{Pretrained ALOHA Unleashed}, \textbf{Action Chunking with Transformers (ACT)} \cite{Zhao-RSS-23}, and \textbf{Pretrained ACT}.
VLA Backbone is a VLA model with the $\pi_{0}$ architecture \cite{black2024pi0visionlanguageactionflowmodel}, pre-trained on the same dataset used for \method pre-training.
Note that this model is trained from scratch and does not use pretrained weights from Physical Intelligence, the company that produced $\pi_{0}$.
ALOHA Unleashed and ACT are state-of-the-art imitation learning methods for bimanual manipulation, and `Pretrained’ denotes pre-training on the same dataset as \method.

\begin{table*}[t]
\centering
\caption{Real-world experiment results comparing \method with baseline.}
\label{tab:real_results}
\setlength{\tabcolsep}{4pt} %
\begin{tabular}{lcccccccccc}
\toprule
\multirow{2}{*}{\textbf{Method}} & \multicolumn{2}{c}{\textbf{Open Drawers}} & \multicolumn{2}{c}{\textbf{Close Drawers}} & \multicolumn{2}{c}{\textbf{Mug on Plate}} & \multicolumn{2}{c}{\textbf{Plate on Rack}} & \multicolumn{2}{c}{\textbf{Recycle Cans}} \\
\cmidrule(lr){2-3} \cmidrule(lr){4-5} \cmidrule(lr){6-7} \cmidrule(lr){8-9} \cmidrule(lr){10-11}
 & Partial & Success & Partial & Success & Partial & Success & Partial & Success & Partial & Success \\
\midrule
ALOHA Unleashed & 2 / 10 & 0 / 10 & \textbf{8 / 10} & 0 / 10 & 4 / 10 & 3 / 10 & 8 / 10 & \textbf{8 / 10} & 9 / 10 & 5 / 10 \\
ALOHA Unleashed with \method & \textbf{8 / 10} & \textbf{7 / 10} & \textbf{8 / 10} & \textbf{5 / 10} & \textbf{8 / 10} & \textbf{8 / 10} & \textbf{9 / 10} & \textbf{8 / 10} & \textbf{10 / 10} & \textbf{6 / 10} \\
\bottomrule
\end{tabular}
% \vspace{-0.5em}
\end{table*}

\begin{figure*}[t]
    \centering
    \includegraphics[width=\textwidth]{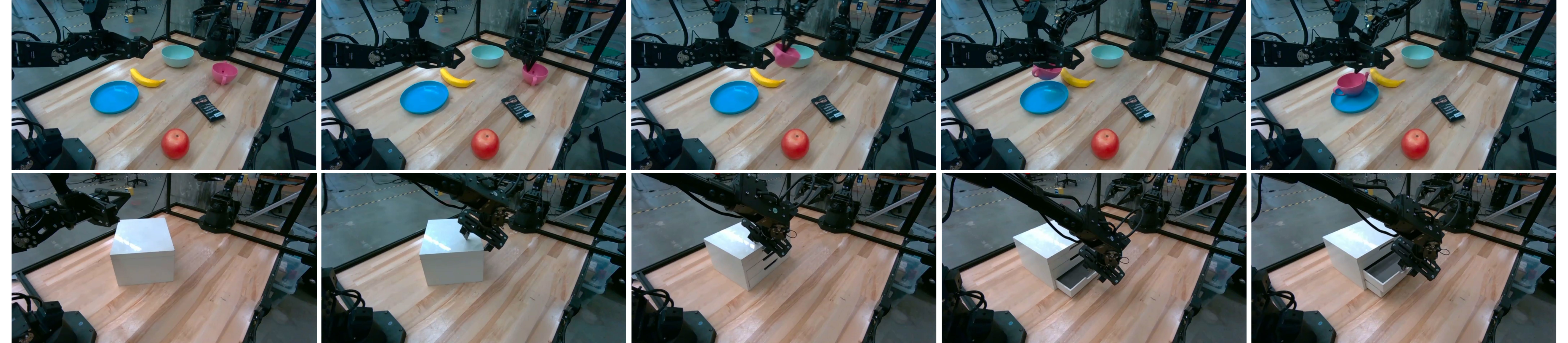}
    \caption{Successful rollouts of our method on the ALOHA robot. Top: \texttt{Mug on Plate}. Bottom: \texttt{Open  Drawers}.}
    \label{fig:real_rollouts}
    % \vspace{-2em}
\end{figure*}

\section{Results}
\label{sec:results}

\subsection{Sim Results}
\autoref{tab:sota_results} reports the maximum success rate across all checkpoints in simulation, averaged over three training seeds, with evaluations conducted over 50 trials every 40k steps.
The top five rows correspond to baseline methods, while the bottom three rows show the same baselines with \method pre-training. ACT with \method pre-training outperforms its non-pretrained counterpart on five of the six tasks. All variants of ALOHA Unleashed with \method pre-training outperform the non-pretrained ALOHA Unleashed across all tasks. Notably, `ALOHA Unleashed with \method' significantly surpasses all baselines on every task.
On the whole, methods with \method pre-training exhibit substantial performance gains, demonstrating the benefits of leveraging learned image and action representations during fine-tuning.
Moreover, we show that pre-training the encoders and the policy in parallel further improves fine-tuning performance.

\autoref{fig:eval_curves} illustrates the evaluation curves of our method (green line) and the baselines, averaged over three training seeds. ALOHA Unleashed with \method pre-training shows strong performance at very early training stages, indicating significant gains in sample efficiency and overall performance due to pre-training. On challenging tasks such as \texttt{Pen in Container}, its performance continues to improve with more training.
% See Appendix \ref{sec:pre_training_results} for the pre-training results.

\begin{table}[t]
\centering
\caption{Ablation experiment results in simulation, averaged over three seeds.}
\label{tab:ablation_experiments}
\setlength{\tabcolsep}{2.2pt}
\renewcommand{\arraystretch}{1.1}
\small
\begin{tabular}{p{0.46\columnwidth}ccc}
\toprule
\textbf{Variation} & \textbf{Can Opener} & \textbf{Pen in} & \textbf{Mug on} \\
 & \textbf{in Caddy} & \textbf{Container} & \textbf{Plate} \\
\midrule
Without virtual wrist views & -2.6\% & -22.0\% & -28.6\% \\
Standard virtual five views \cite{qian20253d} & -9.3\% & -9.3\% & -6.6\% \\
Pre-training w/o text encoder & +0.7\% & -2.0\% & 0.0\% \\
Replace ViT \cite{dosovitskiy2020image} with DP3 \cite{ze20243d} & -15.3\% & -15.3\% & -7.3\% \\
Without STRING \cite{schenck2025learning} & -2.0\% & -3.3\% & -4.6\% \\
Without pre-training encoders & -9.3\% & -7.3\% & -4.6\% \\
Unfrozen encoders (fine-tuning) & -23.3\% & -16.0\% & -3.3\% \\
ACT with \method & -67.4\% & -10.7\% & -59.4\% \\
\bottomrule
\end{tabular}
% \vspace{-2.0em}
\end{table}

\subsection{Real-World Results}

Real-world results are reported in \autoref{tab:real_results}, with example rollouts of our method shown in \autoref{fig:real_rollouts}.
Our method outperforms the baseline on four out of five tasks, while maintaining comparable performance on \texttt{Plate on Rack}.
For \texttt{Open Drawers} and \texttt{Close Drawers}, our method successfully opens and closes both drawers, whereas the baseline struggles with the final drawer. In \texttt{Mug on Plate}, our method shows noticeable improvements in grasping and placing the mug compared to the baseline. For \texttt{Plate on Rack} and \texttt{Recycle Cans}, we observe slightly better grasping performance, while placement performance remains comparable to the baseline. We hypothesize that these tasks are relatively simple and the training data less diverse, limiting gains over the baseline.

% Colab script for this figure: https://colab.corp.google.com/drive/15GK_AwjMb-YRWMOU3NQA5s3tWM4vU_g7#scrollTo=erGrP3A1sj1T
\begin{figure*}[t]
    \centering
    
    % --- Row 1 ---
    \subfloat[Image → Previous Action Recall@5]{
        \includegraphics[width=0.3\linewidth]{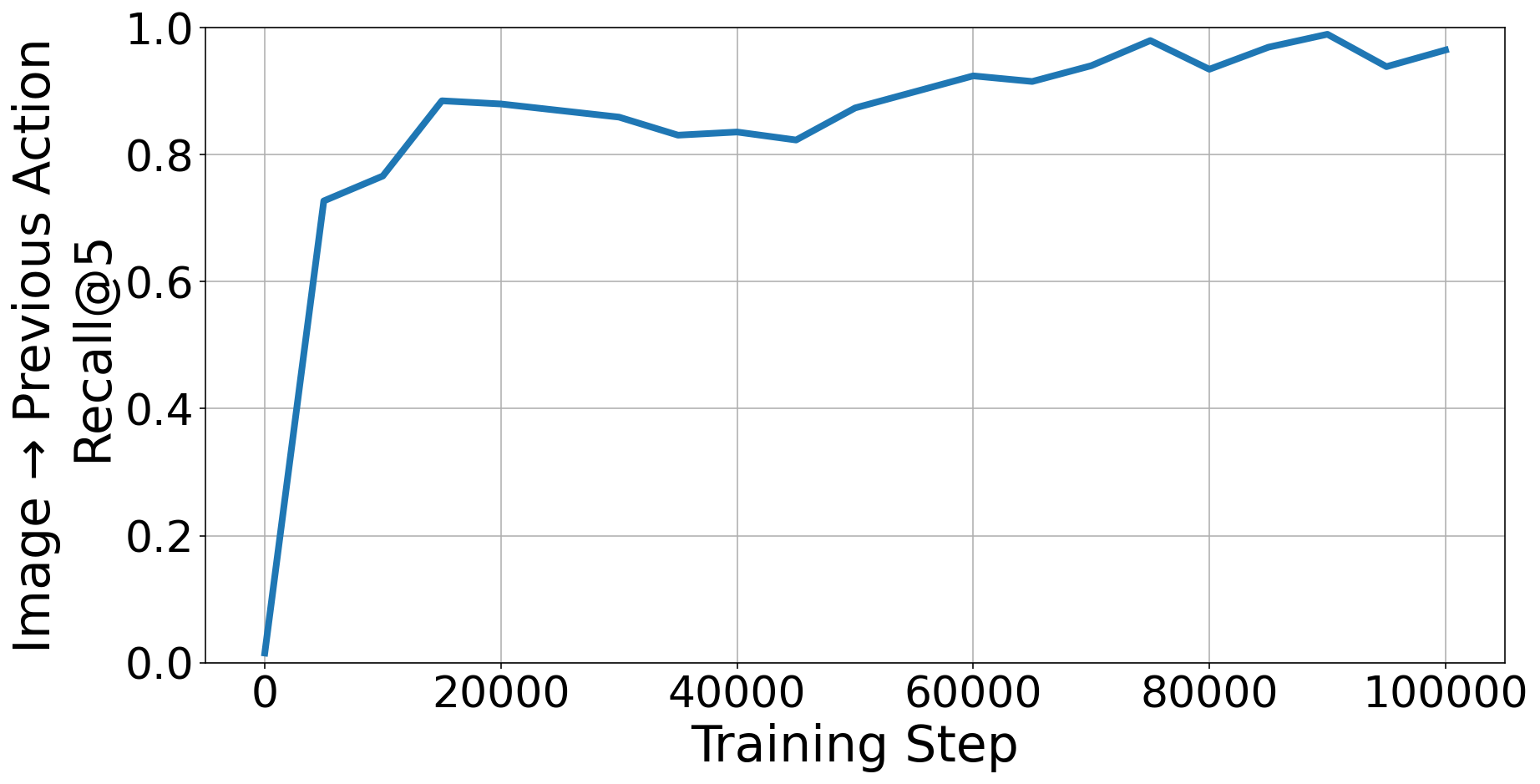}
    }
    \hfill
    \subfloat[Image → Text Recall@5]{
        \includegraphics[width=0.3\linewidth]{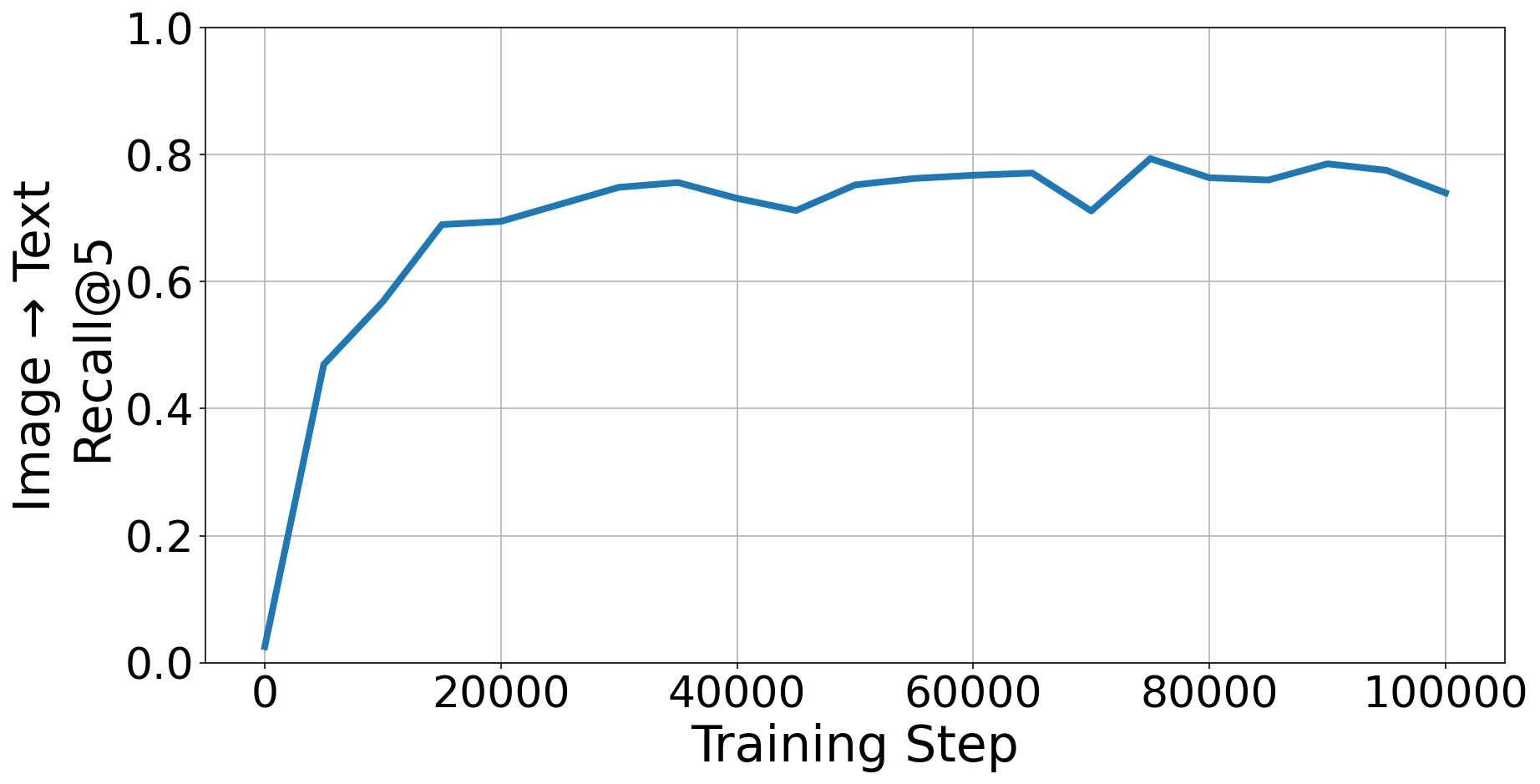}
    }
    \hfill
    \subfloat[Previous Action → Text Recall@5]{
        \includegraphics[width=0.3\linewidth]{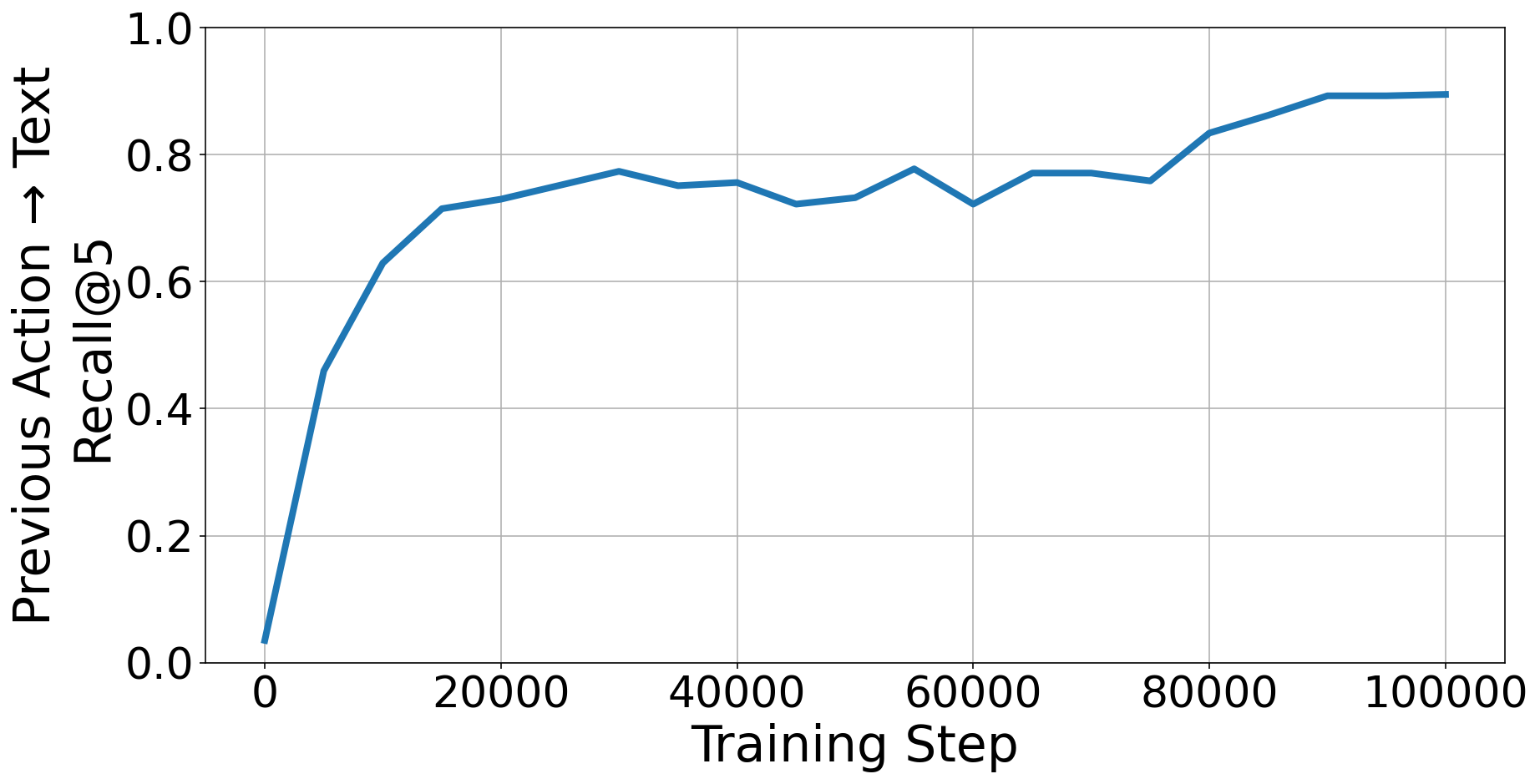}
    }
    
    % --- Row 2 ---
    \subfloat[Previous Action → Image Recall@5]{
        \includegraphics[width=0.3\linewidth]{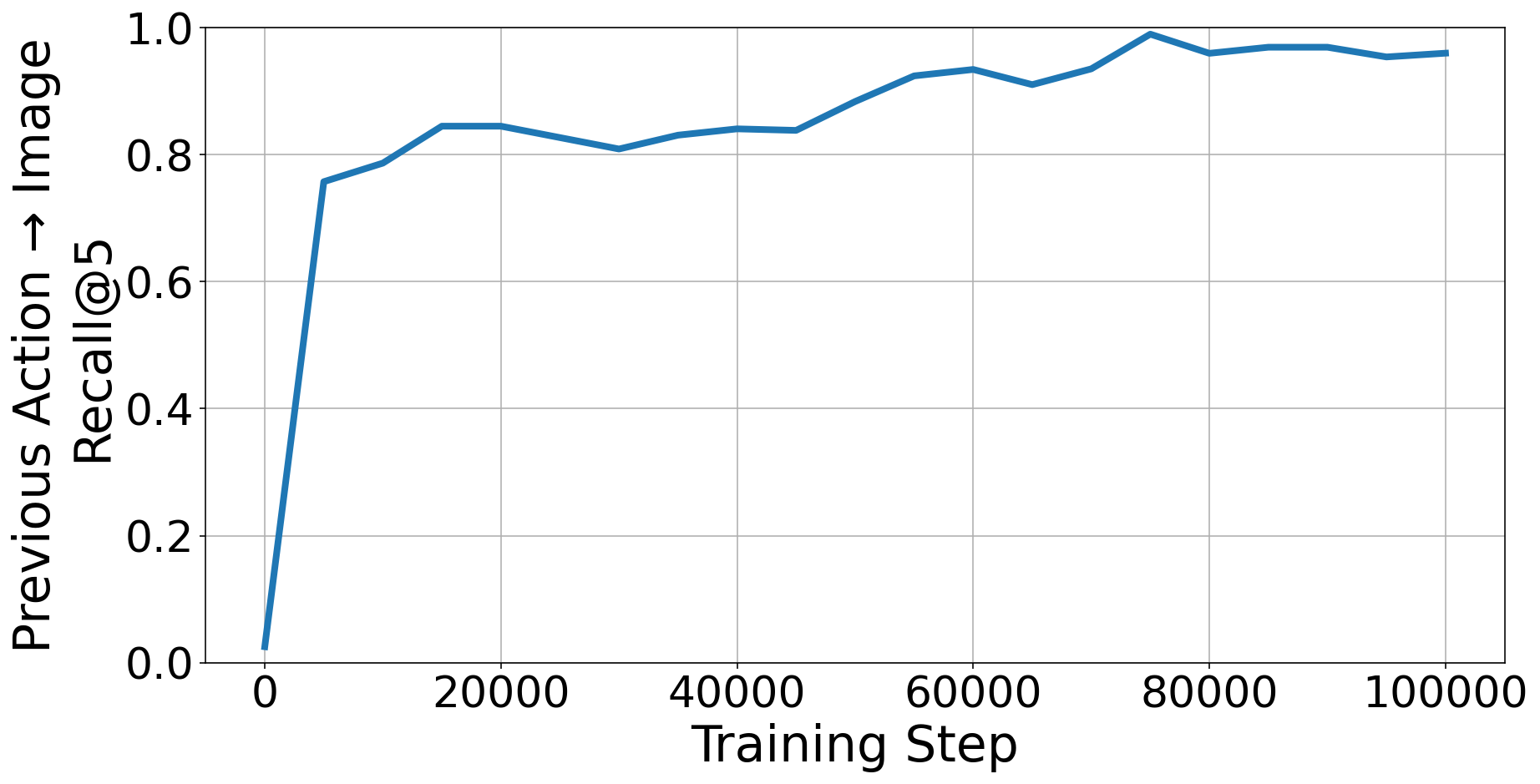}
    }
    \hfill
    \subfloat[Text → Image Recall@5]{
        \includegraphics[width=0.3\linewidth]{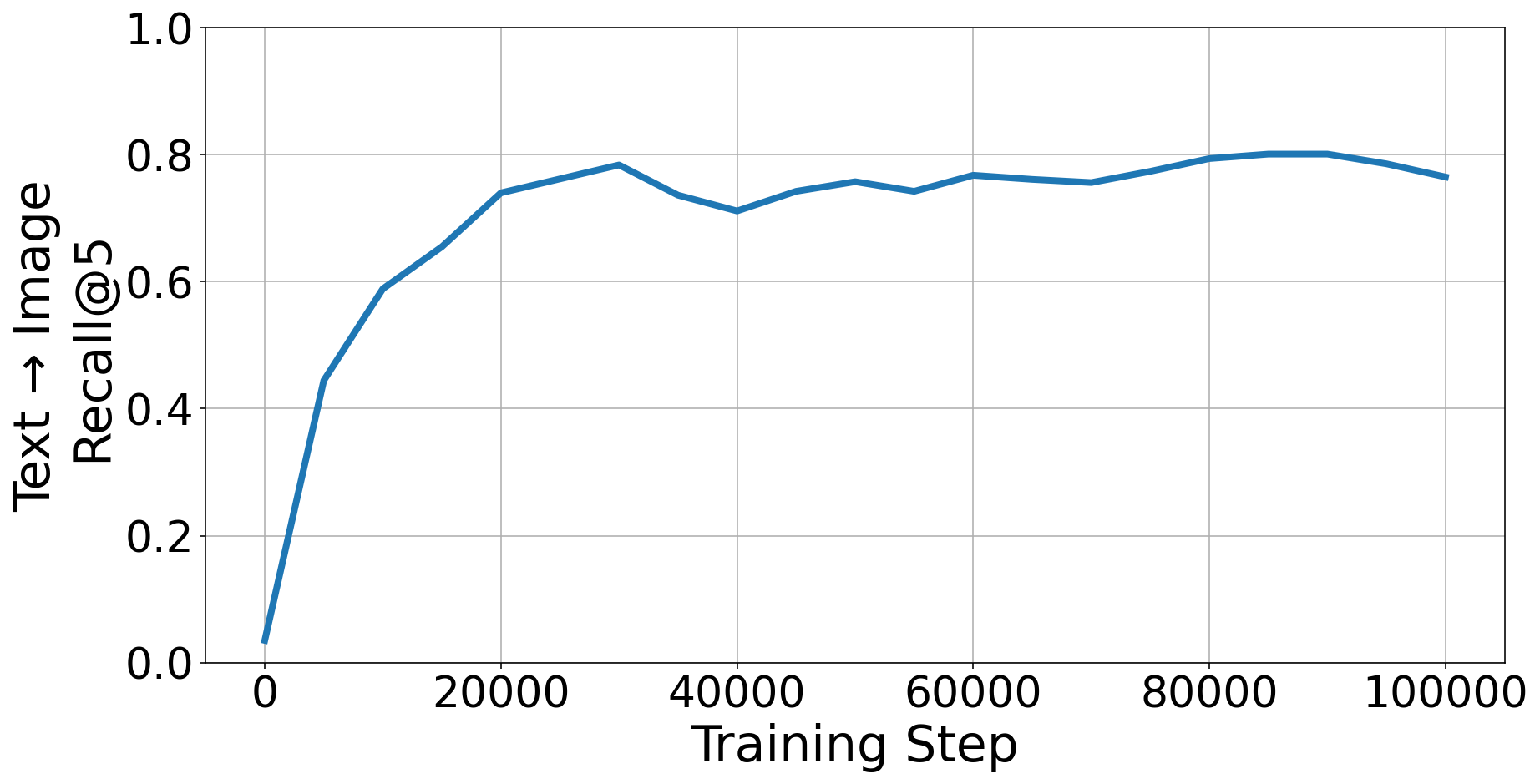}
    }
    \hfill
    \subfloat[Text → Previous Action Recall@5]{
        \includegraphics[width=0.3\linewidth]{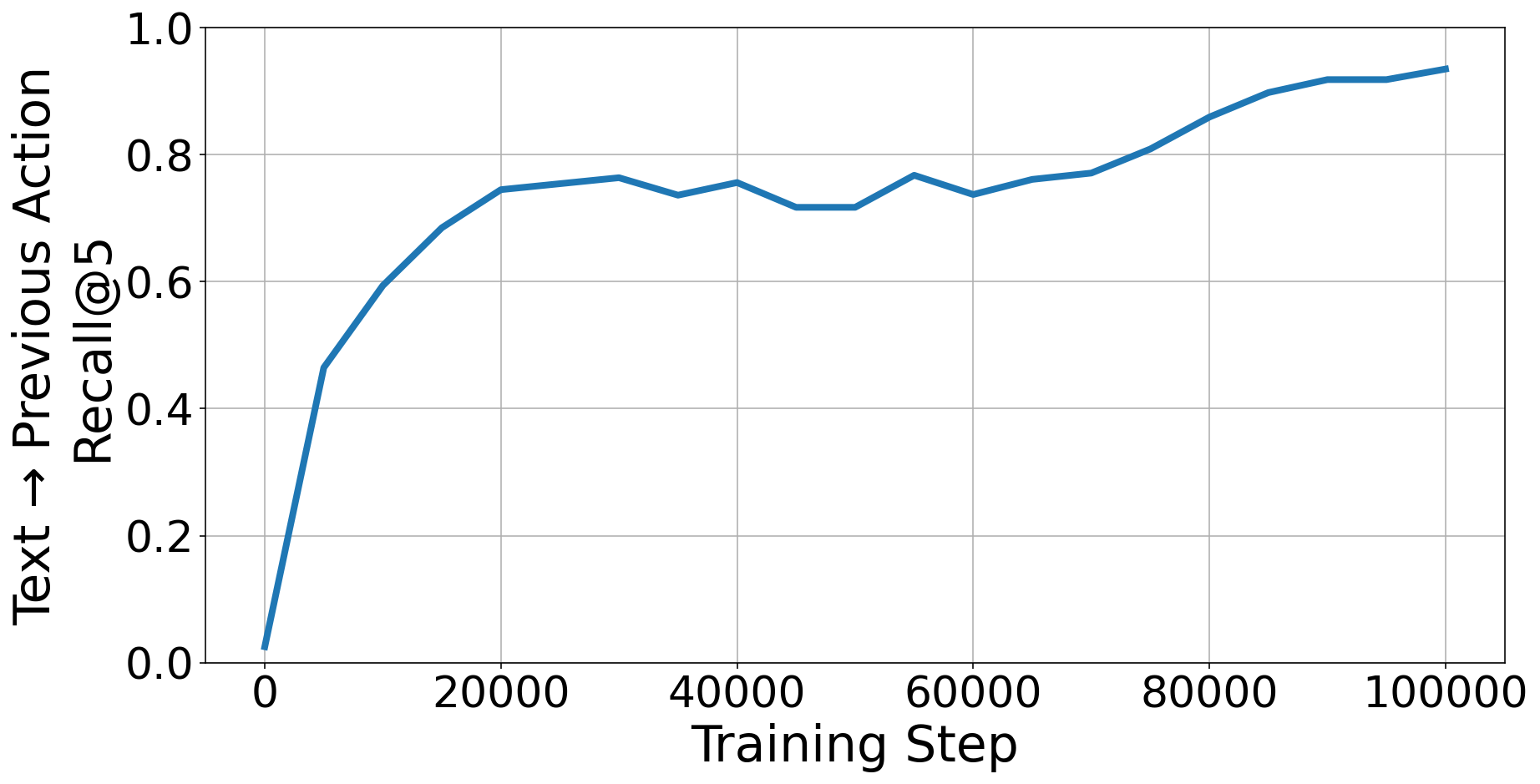}
    }
    
    \caption{Recall@5 results for cross-modal retrieval tasks on the validation data during pre-training.}
    \label{fig:pretraining_retrieval}
    \vspace{-0.2em}
\end{figure*}

% Colab script for this figure: https://colab.corp.google.com/drive/1vyDi-5BxfRNHyQ3asd_YHIau16idZO0r#scrollTo=rRewCBp9Orqt
\begin{figure*}[t]
    \centering
    
    % --- Top Figure ---
    \subfloat{
        \includegraphics[width=0.7\linewidth]{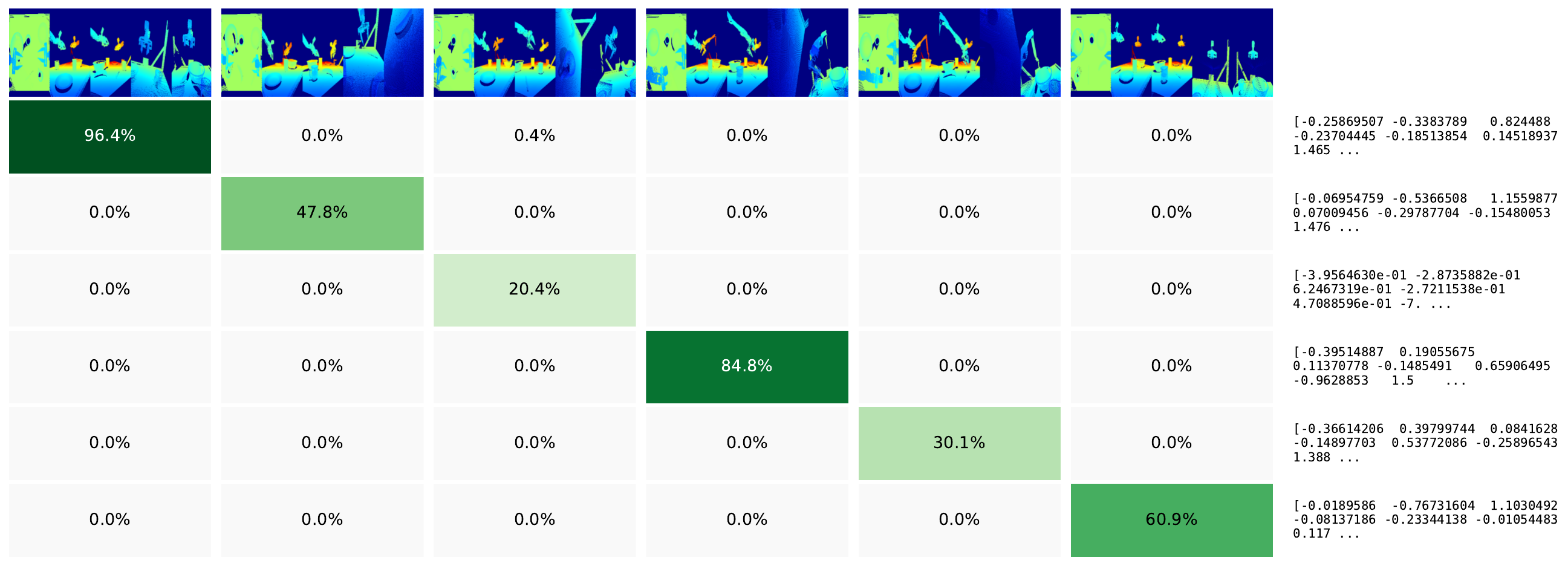}
    }
    
    % This forces a line break to stack them vertically
    \vspace{0.05cm} 
    
    % --- Bottom Figure ---
    \subfloat{
        \includegraphics[width=0.7\linewidth]{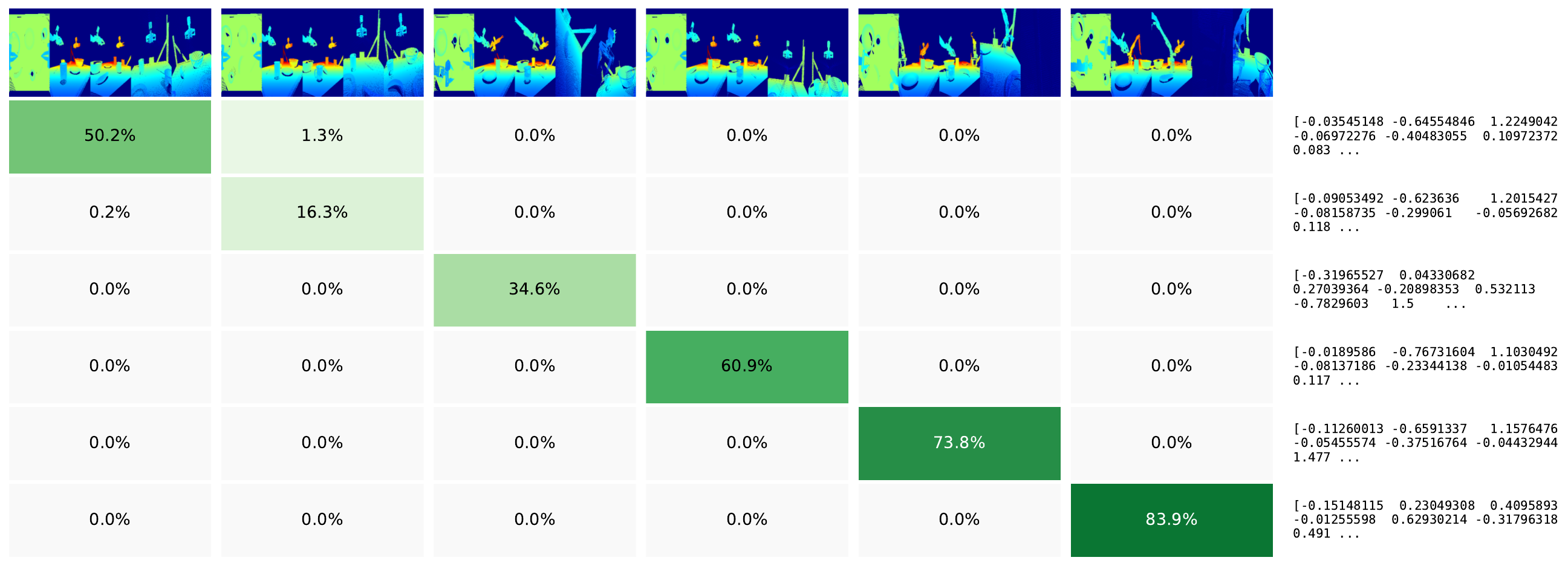}
    }
    
    \caption{SigLIP pairwise matching probabilities on an unseen task: image-to-previous-action (top) and previous-action-to-image (bottom). \method correctly retrieves all image–previous-action pairs on both tasks.}
    \label{fig:pretraining_visualization}
    \vspace{-1.0em}
\end{figure*}

\subsection{Ablations}
\label{sec:ablations}

We ablate key design choices of \method in \autoref{tab:ablation_experiments}. Experiments are conducted in simulation on \texttt{Can Opener in Caddy}, \texttt{Pen in Container}, and \texttt{Mug on Plate}, and results are averaged over three training seeds. As shown in the first two rows, incorporating virtual wrist views substantially improves performance, especially on high-precision tasks such as \texttt{Pen in Container} and \texttt{Mug on Plate}. Moreover, the standard virtual five-view setup used in related work (e.g., \cite{qian20253d}) underperforms our virtual five-view configuration.

In row three, we show that including a text encoder during pre-training yields slightly better performance in simulation,  but we observe that the text encoder noticeably improves real-world performance. We hypothesize that this discrepancy stems from differences between the environments. In simulation, object shapes are consistent and noise-free, so the model relies less on textual information describing the objects. In contrast, objects captured by real-world cameras exhibit texture variations, lighting changes, and other forms of visual noise, making textual information more valuable.

In the following, we ablate the image encoder. Replacing the ViT encoder with a point-cloud encoder (DP3) results in worse performance. Removing STRING positional encoding from ViT, which eliminates learned correlations between tokens from different image patches, also degrades performance. As shown in rows six and seven, it is preferable to pre-train the encoders and keep them frozen during fine-tuning. Finally, the last row shows that ACT with \method performs poorly; we hypothesize this is due to ACT’s limited ability to effectively learn from multimodal demonstrations.

\subsection{Simulation Pre-training Results}
\label{sec:pre_training_results}

% In pc_sim2real/retrieval.py, set a breakpoint to obtain the total number of examples used for validation.  
\autoref{fig:pretraining_retrieval} shows the retrieval results of the \method encoders on six multimodal retrieval tasks on the validation examples. To ensure diversity and accurate assessment, we select examples from the validation set using the following criteria: (1) the time step is a multiple of 10, and (2) the workspace is evenly divided into a $3\times 3$ grid, and the first object’s position has not previously appeared in any grid cell. This results in a total of 3216 examples. As shown in the figure, the \method encoders perform particularly well on tasks that involve retrieving the previous action, highlighting the potential usefulness of the learned action representation. To further examine this representation, we visualize the SigLIP pairwise matching probabilities for image-to-previous-action and previous-action-to-image retrieval on an unseen task in \autoref{fig:pretraining_visualization}. The model correctly retrieves all examples for both tasks, suggesting that the learned image and action representations generalize and can transfer to unseen tasks.

\section{Conclusion}

In this paper, we introduce \method, a 3D action-conditioned pre-training framework for robotic manipulation. \method contrastively pre-trains image, text, and action encoders using 3D multi-view observations, simulator-derived state descriptions, and action histories. The learned image and action representations are then used to fine-tune a Diffusion Policy, which is pre-trained in parallel with the encoders, from limited task demonstrations. Across simulation and real-world experiments, \method substantially improves fine-tuning sample efficiency and final performance over state-of-the-art pretrained and non-pretrained baselines.

\bibliographystyle{plainnat}
\bibliography{references}

\clearpage
\section{Appendix}

\subsection{Proof of Lemma \ref{lemma:cayley-lemma}}
\label{cayley-proof}
\begin{proof}
Note that, by applying the definition of the Cayley-STRING  variant, used in this paper, the not-normalized logits score $s_{i,j}$ between tokens corresponding to $\mathcal{P}_{i}$ and to $\mathcal{P}_{j}$ respectively, is given as:
\begin{equation}
s_{i,j} = (\mathbf{q}^{\prime}_{\mathcal{P}_{i}})^{\top}\mathbf{k}^{\prime}_{\mathcal{P}_{j}},    
\end{equation}
for $\mathbf{q}^{\prime}_{\mathcal{P}_{i}}$ and $\mathbf{k}^{\prime}_{\mathcal{P}_{j}}$ defined as follows:
\begin{equation}
\mathbf{q}^{\prime}_{\mathcal{P}_{i}} = \mathbf{P}\mathrm{RoPE}(\mathbf{r}_{i})\mathbf{P}^{\top}\mathbf{q}_{\mathcal{P}_{i}},
\mathbf{k}^{\prime}_{\mathcal{P}_{j}} = \mathbf{P}\mathrm{RoPE}(\mathbf{r}_{j})\mathbf{P}^{\top}\mathbf{k}_{\mathcal{P}_{j}}
\end{equation}
and where $\mathbf{P}$ is a learnable orthonormal matrix. Here $\mathrm{RoPE}(\mathbf{v})$ stands for the block-diagonal matrix with $\frac{d_{QK}}{2}$ $(2\times2)$-blocks, each given as a $2$-dimensional rotation with the rotation angle defined as: $\delta_{k}(\mathbf{v})$ for the block index $k$ and linear function $\delta_{k}$ (in other words, $\mathrm{RoPE}$ encodes rotation operation consisting of the disjoint Givens rotations; the angles of these Givens rotations are simple linear functions of the positions of the corresponding tokens). Therefore, we have:
\begin{align}
s_{i,j} = \mathbf{q}_{\mathcal{P}_{i}}^{\top}\mathbf{P}\mathrm{RoPE}^{\top}(\mathbf{r}_{i})\mathbf{P}^{\top}\mathbf{P}\mathrm{RoPE}(\mathbf{r}_{j})\mathbf{P}^{\top}\mathbf{k}_{\mathcal{P}_{j}} \\
= \mathbf{v}_{\mathcal{P}_{i}}^{\top} \mathrm{RoPE}(\mathbf{r}_{j}-\mathbf{r}_{i})\mathbf{w}_{\mathcal{P}_{j}}
\end{align}
for vectors $\mathbf{v}_{\mathcal{P}_{i}}$ and $\mathbf{w}_{\mathcal{P}_{j}}$ defined as:
\begin{align}
\begin{split}
\mathbf{v}_{\mathcal{P}_{i}} = \mathbf{P}^{\top}\mathbf{q}_{\mathcal{P}_{i}} \\ 
\mathbf{w}_{\mathcal{P}_{j}} = 
\mathbf{P}^{\top}\mathbf{k}_{\mathcal{P}_{j}} \\ 
\end{split}    
\end{align}
The last equality follows from the fact that $\mathbf{P}^{\top}\mathbf{P}=\mathbf{I}$ (since $\mathbf{P}$ is an orthonormal matrix) and the transpose of the $(2 \times 2)$ rotation matrix is a matrix of the rotation by the negated angle. But this completes the proof since: $\mathbf{r}_{i,j} = \mathbf{r}_{j}-\mathbf{r}_{i}$ and furthermore one can define:
\begin{align}
\begin{split}
\mathbf{q}_{\mathcal{P}_{i}}(\mathcal{L}) = \mathbf{v}_{\mathcal{P}_{i}},    
\mathbf{k}_{\mathcal{P}_{j}}(\mathcal{L}) = \mathbf{w}_{\mathcal{P}_{j}}
\end{split}
\end{align}
(thus the new coordinate system $\mathcal{L}$ is effectively given by the orthonormal matrix $\mathbf{P}^{\top}$).
\end{proof}

% \subsection{Pre-training Results}
% \label{sec:pre_training_results}

% \autoref{fig:pretraining_retrieval} shows the retrieval results of the \method encoders on six multimodal retrieval tasks on the validation examples. To ensure diversity and accurate assessment, we select examples from the validation set using the following criteria: (1) the time step is a multiple of 10, and (2) the workspace is evenly divided into a $3\times 3$ grid, and the first object’s position has not previously appeared in any grid cell. This results in a total of 3216 examples. As shown in the figure, the \method encoders perform particularly well on tasks that involve retrieving the previous action, highlighting the potential usefulness of the learned action representation. To further examine this representation, we visualize the SigLIP pairwise matching probabilities for image-to-previous-action and previous-action-to-image retrieval on an unseen task in \autoref{fig:pretraining_visualization}. The model correctly retrieves all examples for both tasks, suggesting that the learned image and action representations generalize and can transfer to unseen tasks.

\begin{figure}[t]
    \centering
    \includegraphics[width=\columnwidth]{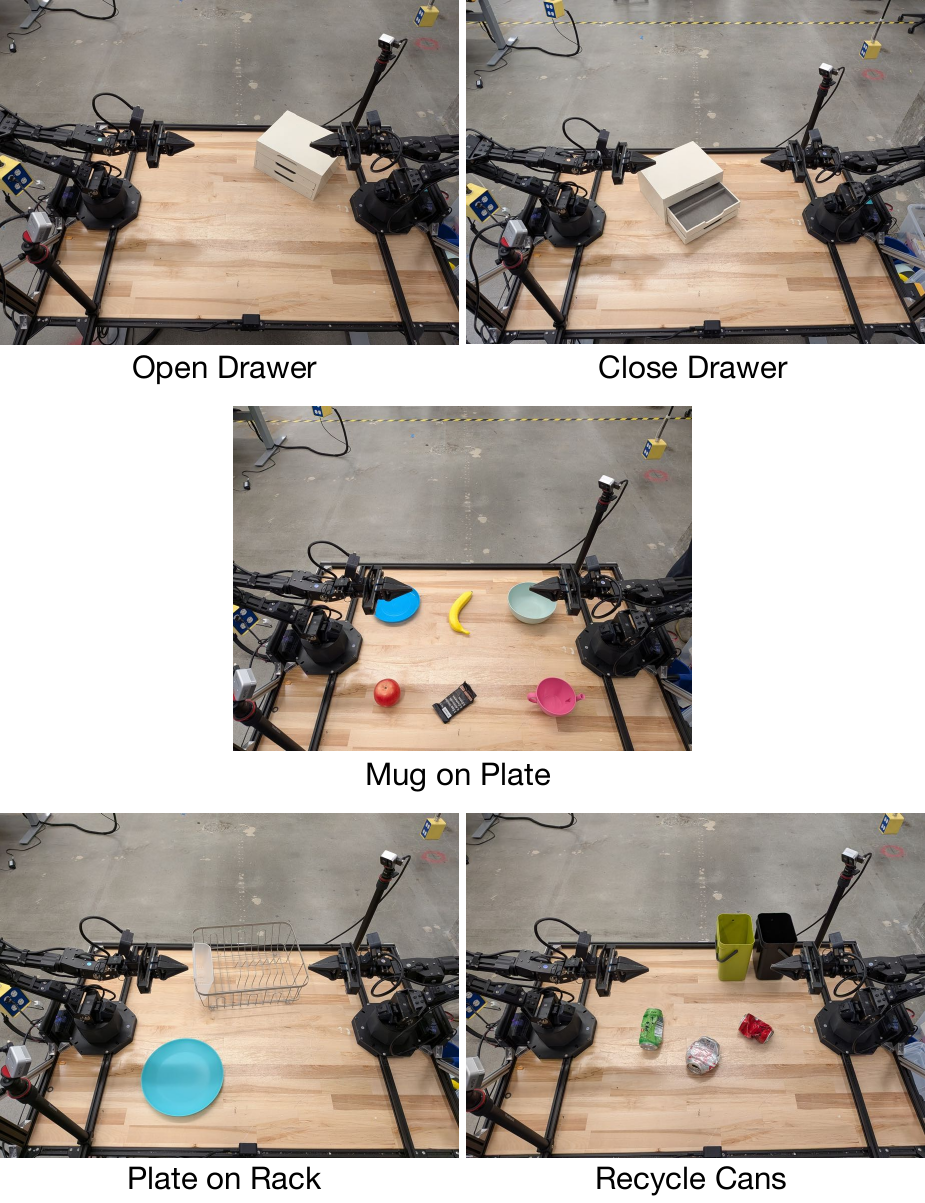}
    \caption{Examples of initial configurations for the real-world tasks.}
    \label{fig:real_world_environments}
    \vspace{-0.2em}
\end{figure}

% \begin{figure*}[t]
%     \centering
    
%     \subfloat[Image → Previous Action Recall@5]{
%         \includegraphics[width=0.3\linewidth]{images/retrieval_image_action.png}
%     }
%     \hfill
%     \subfloat[Image → Text Recall@5]{
%         \includegraphics[width=0.3\linewidth]{images/retrieval_image_text.png}
%     }
%     \hfill
%     \subfloat[Previous Action → Text Recall@5]{
%         \includegraphics[width=0.3\linewidth]{images/retrieval_action_text.png}
%     }
    
%     \subfloat[Previous Action → Image Recall@5]{
%         \includegraphics[width=0.3\linewidth]{images/retrieval_action_image.png}
%     }
%     \hfill
%     \subfloat[Text → Image Recall@5]{
%         \includegraphics[width=0.3\linewidth]{images/retrieval_text_image.png}
%     }
%     \hfill
%     \subfloat[Text → Previous Action Recall@5]{
%         \includegraphics[width=0.3\linewidth]{images/retrieval_text_action.png}
%     }
    
%     \caption{Recall@5 results for cross-modal retrieval tasks on the validation data during pre-training.}
%     \label{fig:pretraining_retrieval}
% \end{figure*}

\subsection{Implementation Details}
\label{sec:implementation_details}

\textbf{Contrastive Pre-training.} We implement the pre-training pipeline in JAX \cite{jax2018github} using the Big Vision framework \cite{big_vision}. We train for 100,000 steps on 64 TPUv4 chips with a batch size of 2048. For image preprocessing, we use a voxel size of 0.001 for voxel-grid downsampling. If the point cloud contains more than 0.3M points, we randomly subsample to 0.3M; otherwise, we pad with points at the origin (0,0,0) to reach 0.3M points. Each virtual view has dimensions $224 \times 224 \times 4$, and we tile the views horizontally to form a $224 \times 1120 \times 4$ input. The action encoder uses a feed-forward dimension of 2048 and a dropout rate of 0.1. Both the image and text encoders use a 3072-dimensional MLP. We use the Adam \cite{KingmaB14} optimizer with $\text{b2}=0.999$ and a cosine-decay learning-rate schedule, a base learning rate of $1e-4$ and 5\% warm-up steps during training. 

\textbf{Diffusion Policy Pre-training and Fine-Tuning.} We implement Diffusion Policy in JAX using an internal robot learning framework. We train for 1M steps on an NVIDIA H100 GPU with a batch size of 8. We use the Adam optimizer with a weight decay of $1e-4$ and a linear learning-rate schedule from 0.0 to $1e-4$ over 5000 transition steps for pre-training and from 0.0 to $1e-5$ over the same transition steps for fine-tuning. We use $848\times480\times3$ RGB images as input to ResNet and apply photometric distortions. The Transformer encoder and decoder use a feed-forward dimension of 3200 and a dropout rate of 0.1. Observations are normalized using the mean and standard deviation computed from the pre-training data, and actions are normalized using the minimum and maximum values from the same dataset.

% \begin{table}[t]
% \centering
% \caption{Additional ablation studies.}
% \label{tab:additional_ablation_design_choices}
% \small
% \setlength{\tabcolsep}{6pt}
% \renewcommand{\arraystretch}{1.1}
% \begin{tabular}{p{0.55\columnwidth} r}
% \toprule
% \textbf{Variation} & \textbf{Can Opener In Caddy} \\
% \midrule
% Image random cropping & -2.0\% \\
% Pre-training decoder only & -10.0\% \\
% Dome images & -42.0\% \\
% Pooled encoder features  
% (fine-tuning) & -12.0\% \\
% Absolute Cartesian control & 0.0\% \\
% Relative Cartesian control & -40.0\% \\
% \bottomrule
% \end{tabular}
% \end{table}

% \begin{figure*}[t]
%     \centering
    
%     \subfloat{
%         \includegraphics[width=0.9\linewidth]{images/visualization_image_action.pdf}
%     }
    
%     \vspace{0.2cm} 
    
%     \subfloat{
%         \includegraphics[width=0.9\linewidth]{images/visualization_action_image.pdf}
%     }
    
%     \caption{SigLIP pairwise matching probabilities on an unseen task: image-to-previous-action (top) and previous-action-to-image (bottom). \method correctly retrieves all image–previous-action pairs on both tasks.}
%     \label{fig:pretraining_visualization}
%     \vspace{-1.0em}
% \end{figure*}

\begin{figure*}[t] %
    \centering
    
    \subfloat{
        \includegraphics[width=0.3\linewidth]{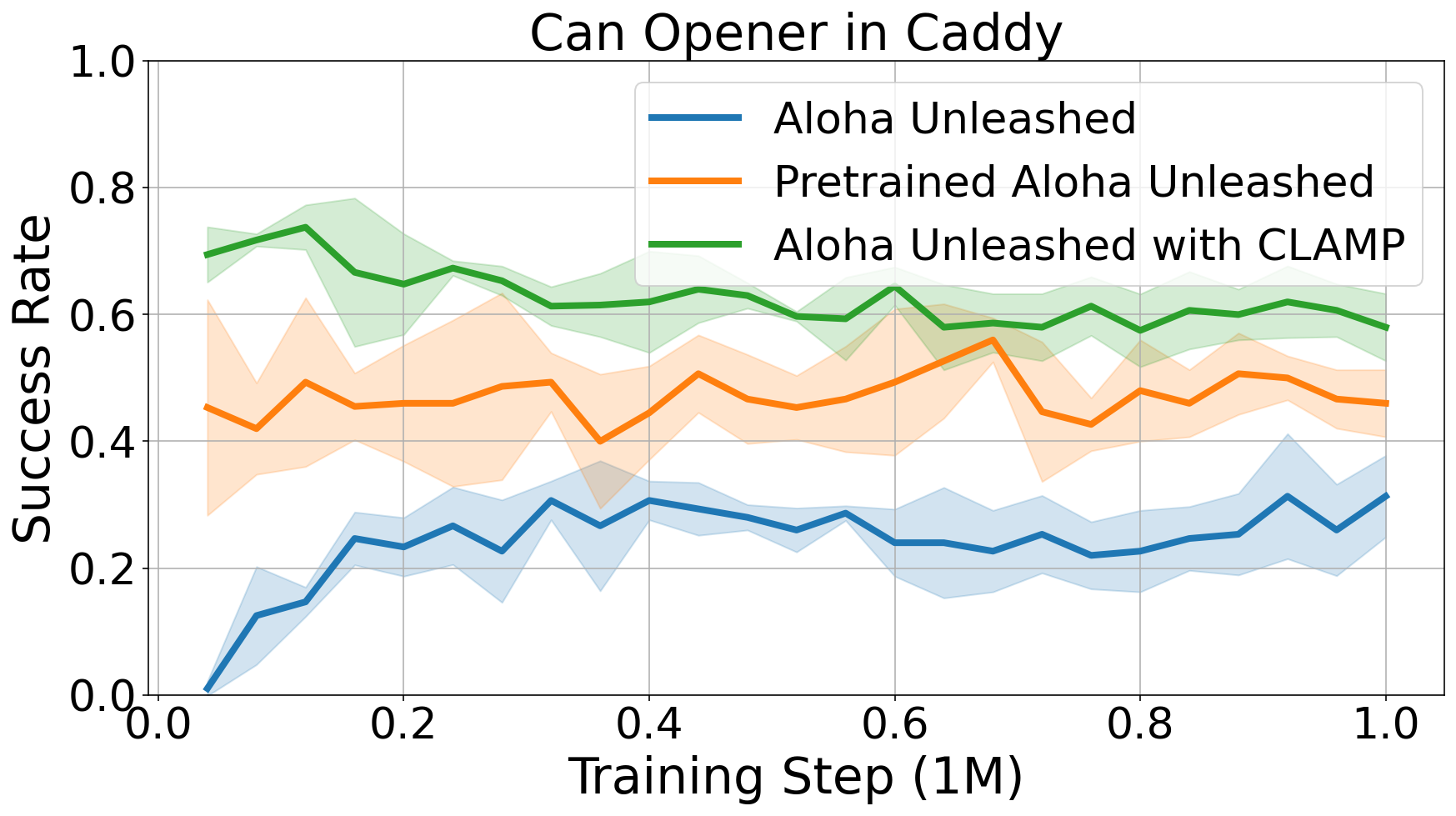}
    }
    \hfill
    \subfloat{
        \includegraphics[width=0.3\linewidth]{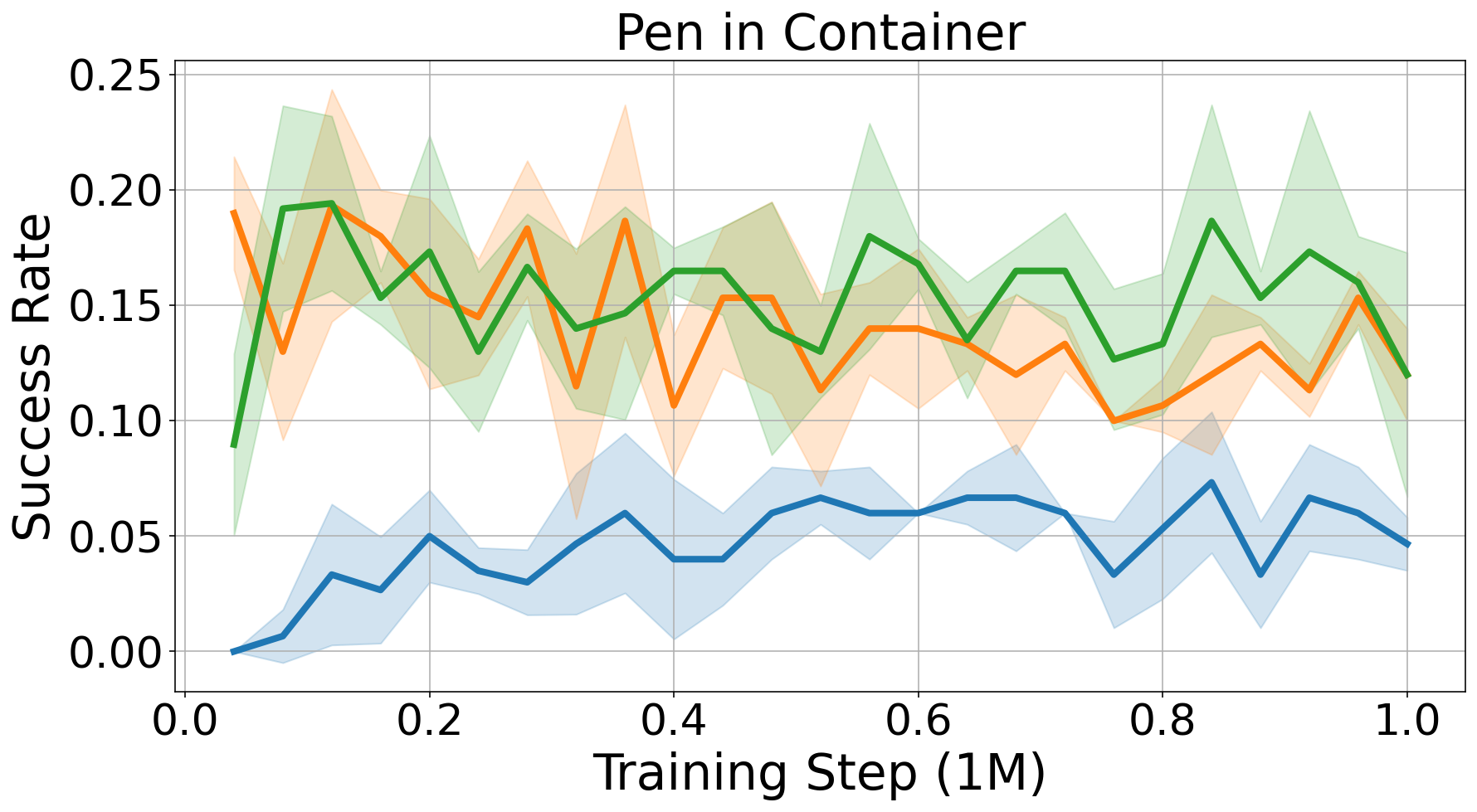}
    }
    \hfill
    \subfloat{
        \includegraphics[width=0.3\linewidth]{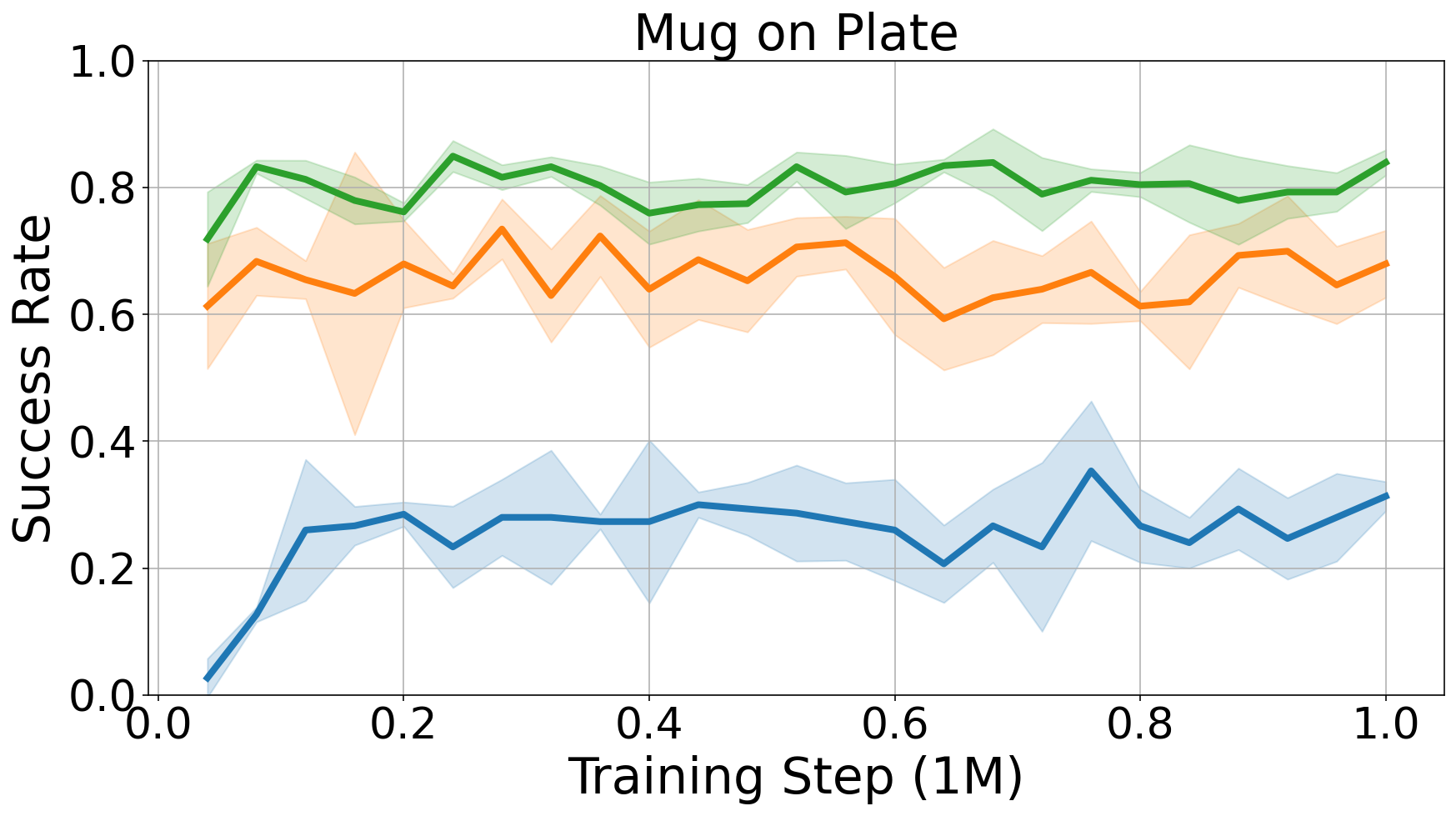}
    }
    
    \caption{Evaluation curves comparing our method (green) with baselines across three seeds in simulation under a low-data regime. All methods are trained on 100 demonstrations for 1M environment steps and evaluated over 50 trials per checkpoint every 40K steps. Note that the y-axis ticks in the \texttt{Pen in Container} plot differ for improved visualization.}
    \label{fig:eval_curves_low_data}
\end{figure*}

\subsection{Additional Experimental Details}

\textbf{Replacing the ViT encoder with DP3 (\autoref{tab:ablation_experiments}, row four)}. DP3 takes as input a processed point cloud that is cropped and voxel-grid downsampled before being re-rendered from virtual viewpoints. This ablation demonstrates that the structured image representation used in our multi-view setup outperforms the unstructured point-cloud representation, and that our multi-view design is essential for achieving strong performance. Moreover, we also experimented with a point cloud Transformer and found that it required significant downsampling of the point cloud to maintain the same batch size used in our multi-view architecture, while a sufficiently large batch size is important for contrastive learning.

\subsection{Additional Experimental Results}

\autoref{tab:additional_ablation_design_choices} presents additional ablations of \method on \texttt{Can Opener in Caddy} in simulation. In the first row, we apply random cropping to images during training and center cropping during inference, which results in worse performance. In the second row, instead of pre-training the entire Diffusion Policy, we pre-train only the Transformer decoder, which also degrades performance. In the third row, we evaluate a 16-view setup constructed using a dome-like structure, where each camera looks at the center of the workspace and the final image has dimensions of $96 \times 1536 \times 4$. However, this performs worse than the five-view setup proposed in the paper. In the fourth row, instead of using unpooled features from the frozen \method encoders, we feed MAP-pooled features (together with ResNet and proprioceptive features) into the Transformer encoder; this highlights the importance of using unpooled features from the frozen encoders. In the last two rows, we evaluate two additional action spaces: absolute Cartesian control and relative Cartesian control. We find that absolute Cartesian control performs similarly to absolute joint-space control, whereas relative Cartesian control performs worse.

\begin{table}[t]
\centering
\caption{Additional ablation studies.}
\label{tab:additional_ablation_design_choices}
\small
\setlength{\tabcolsep}{6pt}
\renewcommand{\arraystretch}{1.1}
\begin{tabular}{p{0.55\columnwidth} r}
\toprule
\textbf{Variation} & \textbf{Can Opener In Caddy} \\
\midrule
Image random cropping & -2.0\% \\
% Only pre-training policy's Transformer action decoder & -10.0\% \\
Pre-training decoder only & -10.0\% \\
Dome images & -42.0\% \\
Pooled encoder features  
(fine-tuning) & -12.0\% \\
Absolute Cartesian control & 0.0\% \\
Relative Cartesian control & -40.0\% \\
\bottomrule
\end{tabular}
\end{table}

\autoref{fig:eval_curves_low_data} shows the evaluation curves of our method (green line) and the baselines, averaged over three training seeds, under a low-data regime. ALOHA Unleashed with \method outperforms all methods in terms of learning efficiency and maximum success rate across all checkpoints. Compared to the same variant trained on more data (\autoref{fig:eval_curves}), this variant performs suboptimally, suggesting that \method fine-tuning is best suited to settings where the fine-tuning dataset is sufficiently large and diverse.

\subsection{Dataset Generation Details}
\label{appendix:sim_dataset}

\subsubsection{GR1.5 tasks and episode counts}
\begin{itemize}
  \item \texttt{multitoolsmagnifierincaddy\_left}: 47{,}679
  \item \texttt{multitoolsscrewdriverincaddy\_left}: 42{,}466
  \item \texttt{multitoolsmagnifierincaddy\_right}: 47{,}613
  \item \texttt{handoverpen}: 39{,}603
  \item \texttt{multitoolsscissorsincaddy\_right}: 42{,}167
  \item \texttt{multitoolscanopenerincaddy\_right}: 42{,}069
  \item \texttt{multitoolsscissorsincaddy\_left}: 40{,}568
  \item \texttt{multitoysraccooninbasket}: 40{,}409
  \item \texttt{multitoysfireengineinbasket}: 31{,}745
  \item \texttt{multitoystraininbasket}: 22{,}867
  
  \item \texttt{multidiningbananainbowl\_vislighting}: 9{,}469
  \item \texttt{multidiningbananainbowl\_visbackground}: 8{,}519
  \item \texttt{multilunch\_redgrape\_right}: 8{,}342
  \item \texttt{multidiningbananainbowl}: 7{,}890
  \item \texttt{multibanana\_camperturb\_rot10deg}: 7{,}871
  \item \texttt{multilunch\_greengrape\_left}: 7{,}409
  \item \texttt{multifruitteabottleinbowl}: 6{,}780
  \item \texttt{marker\_camperturb\_intrin10}: 6{,}444
  \item \texttt{multitools2\_screwdriver\_left}: 5{,}965
  \item \texttt{laptopclose}: 5{,}818
  \item \texttt{microwaveclose}: 5{,}642
  \item \texttt{marker\_camperturb\_rot10deg}: 5{,}608
  \item \texttt{multidiningbananainbowl\_reasoning1}: 5{,}454
  \item \texttt{multitools2\_hammer\_left}: 5{,}421
  \item \texttt{multilunch\_redgrape\_left}: 5{,}327
  \item \texttt{multilunch\_greengrape\_right}: 5{,}248
  \item \texttt{glassonrack}: 5{,}175
  \item \texttt{multibanana\_camperturb\_pos10cm}: 4{,}965
  \item \texttt{multitools2\_hammer\_right}: 4{,}932
  \item \texttt{multibanana\_camperturb\_intrin10}: 4{,}854
  \item \texttt{fruitbowl}: 4{,}853
  \item \texttt{multibanana\_camperturb\_pos5cm}: 4{,}692
  \item \texttt{bowlonrack}: 4{,}380
  \item \texttt{multidiningbananainbowl\_control}: 4{,}133
  \item \texttt{multitoolsscissorsincaddy\_right}: 4{,}118
  \item \texttt{multidiningbananainbowl\_reasoning0}: 3{,}924
  \item \texttt{multibanana\_frozencam\_overhead}: 3{,}910
  \item \texttt{multibanana\_frozencam\_bothextcams}: 1{,}267
  \item \texttt{marker\_camperturb\_pos5cm}: 468
  \item \texttt{marker\_frozencam\_overhead}: 370
  \item \texttt{microwaveopen}: 310
  \item \texttt{multilegoblocksinbag}: 294
  \item \texttt{marker\_frozencam\_bothextcams}: 263
  \item \texttt{marker\_frozencam\_bothwristcams}: 253
  \item \texttt{spooninholder}: 78
  \item \texttt{marker\_camperturb\_pos10cm}: 55
  \item \texttt{utensilsinholder}: 37
  \item \texttt{multibanana\_frozencam\_bothwristcams}: 12
  \item \texttt{multitoysvehiclesinbasket}: 2
\end{itemize}

\end{document}